\documentclass[nohyperref]{article}

\usepackage{microtype}
\usepackage{graphicx}
\usepackage{subfigure}
\usepackage{booktabs} %

\usepackage{hyperref}

\usepackage[accepted]{icml2023}

\usepackage{amsmath}
\usepackage{amssymb}
\usepackage{mathtools}
\usepackage{amsthm}
\usepackage{enumitem}

\usepackage[utf8]{inputenc}
\usepackage[T1]{fontenc}
\usepackage{hyperref}
\usepackage{url}
\usepackage{booktabs}
\usepackage{amsfonts}
\usepackage{nicefrac}
\usepackage{microtype}
\usepackage{xcolor}

\usepackage{empheq}

\usepackage{algorithm, algorithmic}
\usepackage{natbib}
\usepackage{graphicx}
\usepackage{subfigure}
\usepackage{multirow}
\usepackage{amsmath, bm, amsthm}
\usepackage{wrapfig}
\usepackage{xfrac}
\usepackage{sidecap}
\usepackage{adjustbox}
\usepackage{bbm}
\definecolor{azure(colorwheel)}{rgb}{0.0, 0.5, 1.0}
\usepackage{xcolor}
\usepackage{cancel}

\usepackage{thmtools}
\usepackage{thm-restate}
\global\long\def\n{\mathbf{n}}%
\global\long\def\divp#1{\nabla\cdot\left(#1\right)}%
\global\long\def\G{\mathcal{G}}%
\global\long\def\L{\mathcal{L}}%
\global\long\def\si{\sigma}%
\global\long\def\cv#1#2{\left\{  #1\,\middle|\,#2\right\}  }%
\let\originalleft\left
\let\originalright\right
\renewcommand{\left}{\mathopen{}\mathclose\bgroup\originalleft}
\renewcommand{\right}{\aftergroup\egroup\originalright}
\global\long\def\X{\mathcal{X}}%
\global\long\def\K{\mathcal{K}}%
\global\long\def\P{\mathcal{P}}%
\global\long\def\L{\mathcal{L}}%
\global\long\def\S{\mathcal{S}}%
\newcommand{\om}{\omega}
\global\long\def\g{\nabla}%
\global\long\def\gr{\nabla}%
\global\long\def\fr#1#2{\frac{#1}{#2}}%
\newcommand{\cd}{\cdot}
\newcommand{\diver}[1]{\g \cdot \left( #1 \right)}
\global\long\def\norm#1{\left\lVert #1\right\rVert }%
\global\long\def\pa{\mathbf{\partial}}%
\global\long\def\EE{\mathbb{E}}%
\global\long\def\RR{\mathbb{R}}%
\global\long\def\eval#1{\left.#1\right|}%

\DeclareMathOperator*{\argmin}{arg\,min}
\newcommand{\pat}{\frac{\pa}{\pa t}}
\global\long\def\fr#1#2{\frac{#1}{#2}}%
\newcommand{\hl}[1]{#1}

\newcommand{\deriv}[2]{\frac{\partial #1}{\partial #2}}
\newcommand{\id}{{\mathrm{id}}}

\newcommand{\mean}{\mathbb{E}}%
\newcommand{\var}{{\rm I\kern-.3em D}}

\renewcommand{\vec}[1]{\mathbf{#1}}

\newcommand{\cond}{\,|\,}
\newcommand{\Normal}{\mathcal{N}}

\newcommand{\inner}[2]{\left\langle #1, #2\right\rangle}
\newcommand{\eps}{\varepsilon}
\newtheorem*{theorem*}{Theorem}
\newtheorem*{proposition*}{Proposition}

\newtheorem*{example*}{Example}
\DeclareMathSymbol{\shortminus}{\mathbin}{AMSa}{"39}

\usepackage{color}

\newcommand{\lagr}{L}
\newcommand{\hamil}{H}

\newcommand{\vheader}{\vspace*{-.14cm}}

\usepackage[capitalize]{cleveref}
\crefname{section}{Sec.}{Sec.}
\crefname{appendix}{App.}{App.}
\crefname{proposition}{Prop.}{Prop.}

\theoremstyle{plain}
\newtheorem{theorem}{Theorem}[section]
\newtheorem{proposition}[theorem]{Proposition}

\theoremstyle{definition}
\newtheorem{definition}[theorem]{Definition}

\theoremstyle{remark}

\usepackage[textsize=tiny]{todonotes}

\icmltitlerunning{Action Matching}

\begin{document}

\twocolumn[
\icmltitle{Action Matching: \\ Learning Stochastic Dynamics from Samples}

\begin{icmlauthorlist}
\icmlauthor{Kirill Neklyudov}{yyy}
\icmlauthor{Rob Brekelmans}{yyy}
\icmlauthor{Daniel Severo}{yyy,ut}
\icmlauthor{Alireza Makhzani}{yyy,ut}
\end{icmlauthorlist}

\icmlaffiliation{yyy}{Vector Institute}
\icmlaffiliation{ut}{University of Toronto}

\icmlcorrespondingauthor{}{k.necludov@gmail.com}
\icmlcorrespondingauthor{}{makhzani@vectorinstitute.ai}

\icmlkeywords{Machine Learning, ICML}

\vskip 0.3in
]

\printAffiliationsAndNotice{} %

\begin{abstract}
Learning the continuous dynamics of a system from snapshots of its temporal marginals is a problem which appears throughout natural sciences and machine learning, including in quantum systems, single-cell biological data, and generative modeling. In these settings, we assume access to cross-sectional samples that are uncorrelated over time, rather than full trajectories of samples. In order to better understand the systems under observation, we would like to learn a model of the underlying process that allows us to propagate samples in time and thereby simulate entire individual trajectories. In this work, we propose Action Matching, a method for learning a rich family of dynamics using only independent samples from its time evolution. We derive a tractable training objective, which does not rely on explicit assumptions about the underlying dynamics and does not require back-propagation through differential equations or optimal transport solvers. Inspired by connections with optimal transport, we derive extensions of Action Matching to learn stochastic differential equations and dynamics involving creation and destruction of probability mass. Finally, we showcase applications of Action Matching by achieving competitive performance in a diverse set of experiments from biology, physics, and generative modeling.
\end{abstract}

\section{Introduction}\label{sec:intro}

Understanding the time evolution of systems of particles or individuals is a fundamental problem appearing across machine learning and
the natural sciences.   In many scenarios, it is expensive or even physically impossible to observe entire individual trajectories.   For example, in quantum mechanics, the act of measurement at a given point collapses the wave function \citep{griffiths2018introduction}, while in biological applications, single-cell RNA- or ATAC- sequencing techniques destroy the cell in question \citep{macosko2015highly, klein2015droplet, buenrostro2015single}.   

Instead, from `cross-sectional' or independent 
samples at various points in time, we would like to learn a 
model which simulates particles such that their density matches that of the observed samples.
The problem of learning stochastic dynamics from marginal samples is variously referred to as learning \textit{population dynamics} \citep{hashimoto2016learning} or as \textit{trajectory inference} \cite{lavenant2021towards}, in contrast to time series modeling where entire trajectories are assumed to be available.   Learning such models to predict entire trajectories 
holds the promise of facilitating simulation of complex chemical or physical systems 
\citep{vazquez2007porous, noe2020machine} 
and understanding developmental processes or treatment effects in biology \citep{schiebinger2019optimal, tong2020trajectorynet, schiebinger2021reconstructing, bunne2021learning}.

\begin{figure*}[t]
    \centering
    \includegraphics[width=0.69\textwidth]{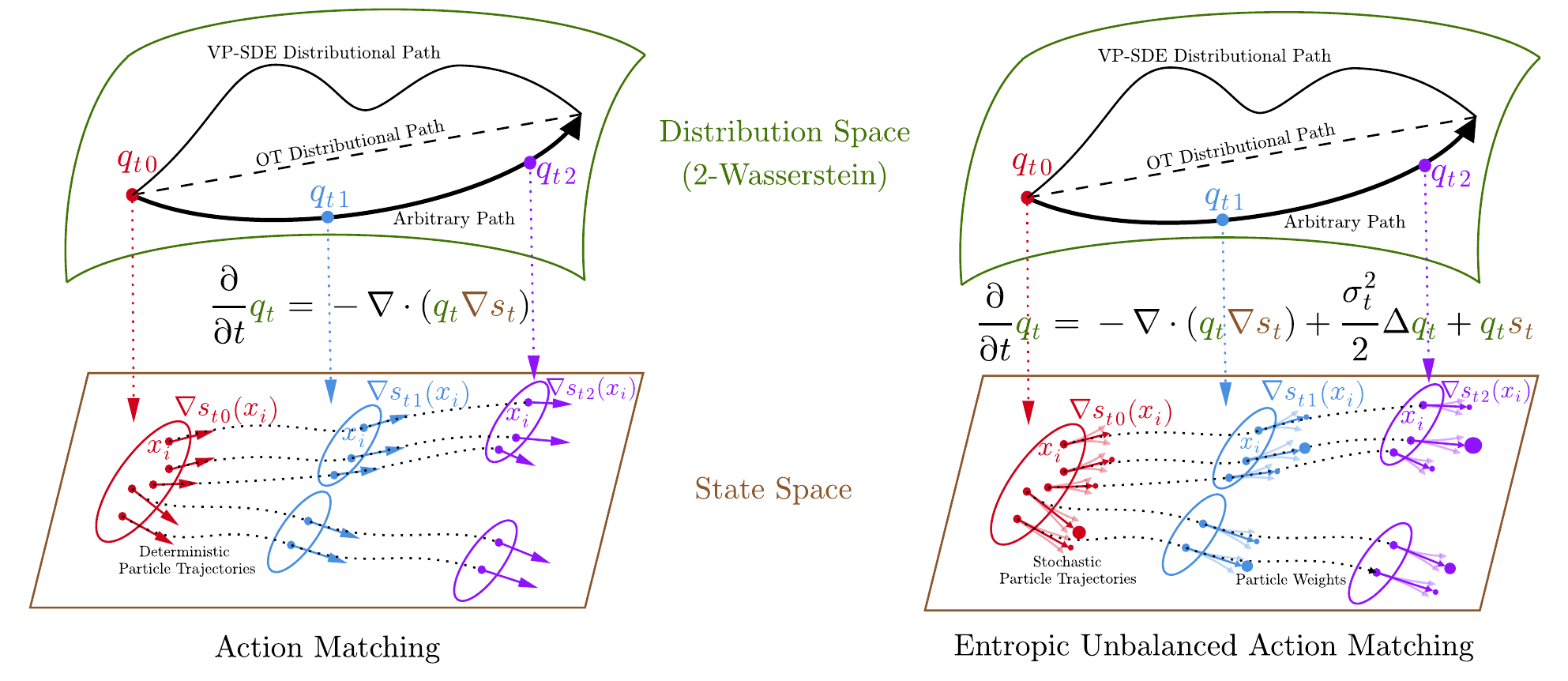}
    \vskip -10pt
    \caption{Action Matching, and its entropic (eAM) and unbalanced (uAM) variants, can learn to trace any arbitrary %
    distributional path. 
    For a given path, AM learns deterministic trajectories, eAM learns stochastic trajectories, and uAM learns weighted trajectories.}
    \label{fig:continuity_eq}
    \vspace*{-.2cm}
\end{figure*}

Furthermore, recent advances in generative modeling have been built upon learning stochastic dynamics which interpolate between the data distribution and a prior distribution.   
In particular, score-based diffusion models \cite{song2020score, ho2020denoising} construct a stochastic differential equation (SDE) to move samples from the data distribution to a prior distribution, while score matching \citep{hyvarinen2005estimation} is used to learn a reverse SDE which models the gradients of intermediate distributions.  
However, these methods rely on analytical forms of the SDEs and/or the tractability of intermediate Gaussian distributions \citep{lipman2022flow}.
Since our proposed method can learn dynamics which simulate an arbitrary path of marginal distributions, it can also be applied in the context of generative modeling. 
Namely, we can approach generative modeling by constructing an interpolating path between the data and an arbitrary prior distribution, and learning to model the resulting dynamics.

In this work, we propose \textit{Action Matching}, a method for learning 
population dynamics 
{from samples of their temporal marginals $q_t$}. 
Our contributions are as follows:
\begin{itemize}[nosep]
    \item  In \cref{th:ac_cont_short}, we establish the existence of a unique gradient field $\nabla s_t^*$ which traces any given time-continuous distributional path $q_t$. Notably, our restriction to gradient fields is without loss of expressivity for this class of $q_t$. To learn this gradient field, we propose the tractable Action Matching training objective in \cref{prop:am_loss}.
    \item In \cref{sec:entropic_am}-\ref{sec:ccosts}, we extend the above approach in several ways: an `entropic' version which can approximate ground-truth dynamics involving stochasticity, an `unbalanced' version which allows for creation and destruction of probability mass, and a version which can minimize an arbitrary convex cost function in the Action Matching objective.
    \item We discuss the close relationship between Action Matching and dynamical optimal transport, along with other related works in \cref{sec:discussion} and \cref{app:am_ot}.
    \item Since Action Matching relies only on samples and does not require tractable intermediate densities or knowledge of the underlying stochastic dynamics, it is applicable in a wide variety of problem settings. 
    In particular, we demonstrate competitive performance of Action Matching in a number of experiments, including trajectory inference in biological data (\cref{sec:bio}), evolution of quantum systems (\cref{sec:schrodinger}), and a variety of tasks in generative modeling (\cref{sec:gen_modeling}).
\end{itemize}

\vspace{-.2cm}
\section{Action Matching}\label{sec:action_matching}
\label{sec:well_defined}
\subsection{Continuity Equation}
Suppose we have a set of particles in space $\X \subset \RR^d$, initially distributed as $q_{t=0}$. Let each particle follow a time-dependent ODE (continuous flow) with the velocity field $v:[0,1]\times \X \to \RR^d$ as follows 
\begin{align}\label{eq:am_ode}
    \frac{d}{dt}  x(t) = v_t(x(t))\,, \quad x(t=0) = x\,.
\end{align}
The continuity equation describes how the density of the particles $q_t$ evolves in time $t$, i.e.,
\begin{align}
  \pat q_t = - \diver{q_t v_t}\,,        
\end{align}
which holds in the distributional sense, where $\g \cd$ denotes the divergence operator.
Under mild conditions, the following theorem shows that \emph{any continuous dynamics can be modeled by the continuity equation}, and moreover any continuity equation results in a continuous dynamics.
\begin{restatable}[Adapted from Theorem 8.3.1 of \citet{ambrosio2008gradient}]{theorem}{amunique}\label{th:ac_cont_short}
    Consider a continuous dynamic with the density evolution of $q_t$, which satisfies mild conditions (absolute continuity in the 2-Wasserstein space of distributions $\P_2(\X)$).
    Then, there exists a unique (up to a constant) function $s_t^*(x)$, called the ``action'',
    \footnote{The Hamilton-Jacobi formulation of classical mechanics describes the velocity of particles using a gradient field of the ``action'', which matches our usage throughout this work.}
    such that vector field  $v_t^*(x)=\g s_t^*(x)$ and $q_t$ satisfies the continuity equation
    \begin{align}
        \pat q_t = -\diver{q_t \g s_t^*(x)}\,.
        \label{eq:cont_eq}
    \end{align}
    In other words, the ODE $\frac{d}{dt} x(t) = \g s_t^*(x)$ can be used to move samples in time such that the marginals are $q_t$.
\end{restatable}
Using \cref{th:ac_cont_short}, the problem of learning the dynamics can be boiled down to learning the unique vector field $\nabla s_t^*$, only using samples from $q_t$. Motivated by this, we restrict our search space 
to the family of gradient vector fields
\begin{align}
    \S_t = \cv{\g s_t}{s_t:\X \to \mathbb{R}}\,.
\end{align}
We use a neural network to parameterize the set of functions $s_t(x; \theta)$, and will propose the Action Matching objective in \cref{sec:action_matching} to learn 
parameters $\theta$ such that
$s_t(x; \theta)$ approximates $s_t^*(x)$. Once we have learned the vector field, 
we can move samples forward or backward in time by simulating the ODE in \cref{eq:am_ode} with the velocity $\nabla s_t$.
The continuity equation ensures that  for $\nabla s_t^*$, samples at any given time 
$t \in [0,1]$
are distributed according to $q_{t}$. 

Note that, even though we arrived at the continuity equation and ground truth vector field $\nabla s_t^*(x)$ using ODEs,
continuity equation can describe a rich family of density evolutions, including diffusion equation (see Equation 37 of \citet{song2020score}), or even more general evolutions such as porous medium equations \citep{otto2001geometry} in fluid mechanics.
Since these processes also define an absolutely continuous curve in the density space, \cref{th:ac_cont_short} applies.   Thus, for the task of modeling the marginal evolution of $q_t$, our restriction to ODEs using gradient vector fields does not sacrifice expressivity.

\vheader
\subsection{Action Matching Loss}\label{sec:am_loss}
The main development of this paper is the Action Matching method, which allows us to recover the \emph{true action} $s_t^*$ while having access only to samples from $q_t$.   With this action in hand, we can simulate the continuous dynamics whose evolution matches $q_t$ using the vector field $\nabla s_t^*$ (see \cref{fig:continuity_eq}).

In order to do so, we define the \emph{variational action} $s_t(x)$ parameterized by a neural network, which approximates $s_t^*(x)$ by minimizing the ``$\textsc{action-gap}$'' objective
\fontsize{8.5pt}{11pt}\selectfont
\begin{align}\label{eq:am_action_gap}
   \textsc{action-gap}(s,s^*) \coloneqq \fr 12\int_0^1
   \EE_{q_{t}\left(x\right)}{\norm{\g s_{t}\left(x\right)-\g s_{t}^{*}\left(x\right)}^{2}}dt  \,.
\end{align}
\normalsize
Note that this objective is intractable, as we do not have access to $\g s^*$. However as the following proposition shows, we can still derive a tractable objective for minimizing the action gap.

\begin{restatable}{theorem}{amloss}\label{prop:am_loss}
For an arbitrary variational action $s$, the $\textsc{action-gap}(s,s^*)$ can be decomposed as the sum of an intractable constant
$\K$,
and a tractable term $\mathcal{L}_{\text{AM}}\left(s\right)$
\begin{align}
    &\textsc{action-gap}(s_t,s_t^*) =  
    \K 
    + \mathcal{L}_{\text{\emph{AM}}}\left(s_t\right) \,.
\end{align}
where $\mathcal{L}_{\text{\emph{AM}}}(s)$ is the Action Matching objective, which we minimize
\begin{align}
\begin{split}
\mathcal{L}_{\text{\emph{AM}}}(s)&\coloneqq \EE_{q_{0}\left(x\right)}\big[s_{0}\left(x\right)\big]-\EE_{q_{{1}}\left(x\right)}\big[s_{{1}}\left(x\right)\big] \\
&+\int_{0}^{1} \EE_{q_{t}\left(x\right)}\left[\fr 12\norm{\g s_{t}\left(x\right)}^{2}+\fr{\pa s_{t}}{\pa t}\left(x\right) \right]dt\,
\end{split}\label{eq:kilbo}
\end{align}
\end{restatable}
See \cref{app:am_proof} for the proof. The term $\mathcal{L}_{\text{AM}}$ is tractable, since we can use the samples from marginals $q_t$ to obtain an unbiased low variance estimate.   We show in \cref{app:am_proof} that the intractable constant $\K$ is the \emph{kinetic energy} of the distributional path, defined as $\K(\g s_t^*)\coloneqq\fr 12\int_{0}^{{1}}\EE_{q_{t}\left(x\right)}\norm{\g s_{t}^{*}(x)}^{2}dt$, and 
thus minimizing $\mathcal{L}_{\text{\emph{AM}}}(s)$ can be viewed as maximizing a variational lower bound on the kinetic energy.
\paragraph{Connection with Optimal Transport}
In \cref{app:infinitesimal}, we show that the optimal dynamics of AM along the curve is also optimal in the sense of optimal transport with the 2-Wasserstein cost.  
More precisely, at any given time $t$, 
the optimal vector field in the AM objective defines a mapping between two infinitesimally close distributions $q_t$ and $q_{t+h}$, which is of the form $x \mapsto x+h \nabla s^*_t(x)$. This mapping is indeed the same as the Brenier map \citep{brenier1987decomposition} in optimal transport, which is of the form $x \mapsto x+\nabla \varphi_t(x)$, where $\varphi_t$ is the (c-convex) Kantorovich potential.

Finally, in \cref{app:am_proof}, we adapt reasoning from \citet{albergo2022building} to show that the 2-Wasserstein distance between the ground truth marginals and those simulated using our learned $\nabla s_t(x)$ can be upper bounded in terms of $\textsc{action-gap}(s_t, s_t^*)$.

\newcommand{\ii}{{i}}
\newcommand{\jj}{{j}}
\newcommand{\sampledt}{{t^i}}
\newcommand{\tforst}{t}
\newcommand{\stt}{{\tforst^i}}
\begin{algorithm}[t]
  \caption{Action Matching \footnotemark}
  \begin{algorithmic}
    \REQUIRE data $\{x^\jj_t\}_{j=1}^{N_t}, \; x^{\jj}_t 
    \sim q_t(x)$ 
    \REQUIRE parametric model $s_t(x,\theta)$
    \FOR{learning iterations}
        \STATE get batch of samples from boundaries:
        \STATE \qquad $\{x_0^\ii\}_{i=1}^n \sim q_0(x), \; \{x_1^\ii\}_{i=1}^n \sim q_1(x)$ \\[1.5ex]
        \STATE sample times $\{\sampledt\}_{i=1}^n \sim \text{Uniform}[0,1]$\\[1.25ex]
        \STATE get batch of intermediate samples $\{x_{\sampledt}^\ii\}_{i=1}^n \sim q_{t}(x)$\\[1.25ex]
        \STATE $\L_{\text{AM}}(\theta) = \frac{1}{n}\sum \limits_{i=1}^n \bigg[s_{0}(x_0^\ii,\theta) - s_{1}(x_1^\ii,\theta) $
        \STATE {$\phantom{\L_{\text{AM}} =(\theta) = \frac{1}{n} } + \frac{1}{2}\norm{\nabla s_{\stt}(x_{{\sampledt}}^\ii,\theta)}^2 + \deriv{s_{\stt}(x_{{\sampledt}}^\ii,\theta)}{\tforst}\bigg]$} \\[1.25ex]
        \STATE update the model $\theta \gets \text{Optimizer}(\theta, \nabla_\theta\L_{\text{AM}}(\theta))$
    \ENDFOR
    \OUTPUT trained model $s_t(x,\theta^*)$
  \end{algorithmic}
  \label{alg:amgl}
\end{algorithm}
\footnotetext{Notebooks with pedagogical examples of AM are given at \href{https://github.com/necludov/jam\#tutorials}{github.com/necludov/jam\#tutorials}.}

\vheader
\subsection{Learning, Sampling, and Likelihood Evaluation}
\paragraph{Learning} 
We provide pseudo-code for learning with the Action Matching objective in \cref{alg:amgl}.
With our learned $\nabla s_t(x, \theta)$, we now describe how to simulate the dynamics 
and evaluate likelihoods 
when
the initial 
density 
$q_0$ is known.
\paragraph{Sampling}
We sample from the target distribution via the trained function $s_t(x(t),\theta^*)$ by solving the following ODE %
forward 
in time:
\begin{align}
    \frac{d}{dt} x(t) = \nabla_x s_t(x(t),\theta^*), \;\;\; x(t=0)
    \sim q_{0}(x). \label{eq:sampling}
\end{align}
Recall that this sampling process is justified by \cref{eq:cont_eq}, where $s_t(x(t),\theta^*)$ approximates $s_t^*(x(t))$.
\paragraph{Evaluating the Log-Likelihood} 
When the density for $q_0$ is available, we can evaluate the log-likelihood of a sample $x \sim q_1$ using the continuous change of variables formula \citep{chen2018neural}. Integrating the ODE backward in time,
\begin{equation}
\begin{aligned}
    \log q_1(x) &= \log q_0(x(0)) - \int_0^1 dt\; \Delta s_t^*(x(t)), \\
    \frac{d}{dt} x(t) &= \nabla_x s_t^*(x(t)),\quad
    x(t=1) = x, \label{eq:lkd_eval}
\end{aligned}
\end{equation}
where $\frac{d}{dt}\log q_t = -\Delta s_t^*$ can be confirmed using a simple calculation and we approximate $s_t^*(x(t))$ by $s_t(x(t),\theta^*)$. 

\vheader
\section{Extensions of Action Matching \protect\footnote{This section can be skipped without loss of understanding of our core contributions, although our experiments also evaluate the entropic \textsc{am} method from \cref{sec:entropic_am}.}}

In this section, we propose several extensions of Action Matching, which can be used to learn dynamics which include stochasticity (\cref{sec:entropic_am}), allow for teleportation of probability mass (\cref{sec:unbalanced_am}), and minimize alternative kinetic energy costs (\cref{sec:ccosts}).

\vheader
\subsection{Entropic Action Matching}\label{sec:entropic_am}

In this section, we propose \textit{entropic} Action Matching (eAM), which can recover the ground-truth dynamics arising from diffusion processes with gradient field drift term and known diffusion term.  
This setting takes place in biological applications 
studying the Brownian motion of cells in a medium \citep{schiebinger2019optimal, tong2020trajectorynet}.
We will show in \cref{prop:entropic_am} that, at optimality,
entropic AM can \textit{also} learn 
{to trace any absolutely continuous distributional path under mild conditions}, so that the choice between entropic AM and deterministic AM should be made based on prior knowledge of the true underlying dynamics.
Consider the stochastic differential equation
\begin{align}
    dx(t) = v_t(x)dt+\sigma_t dW_t\,, \quad x(t=0) = x\,.
\end{align}
where $W_t$ is the Wiener process. We know that the evolution of density of this diffusion process is described by the Fokker–Planck equation:
\begin{align}
  \pat q_t = - \diver{q_t v_t} + \frac{\sigma_t^2}{2} \Delta q_t\,,        
\end{align}
In the following proposition, we extend \cref{th:ac_cont_short} and prove that any continuous distributional path, regardless of ground truth generating dynamics, can be modeled with the diffusion dynamics in the state-space.

\begin{restatable}{proposition}{eamexist}\label{prop:entropic_am}
    Consider a continuous dynamic with the density evolution of $q_t$, and suppose $\si_t$ is given.
    Then, there exists a unique (up to a constant) function $\tilde{s}_t^*(x)$, called the ``entropic action'', such that vector field  $v_t^*(x)=\g \tilde{s}_t^*(x)$ and $q_t$ satisfies the Fokker-Planck equation
    \begin{align}
        \pat q_t = - \diver{q_t \nabla \tilde{s}_t^*} + \frac{\sigma_t^2}{2} \Delta q_t\,, 
        \label{eq:fokker_planck_eq}
    \end{align}
\end{restatable}
See \cref{app:eam_proof} for the proof. This proposition indicates that the we can use the the SDE $dx(t) = \g \tilde{s}_t^*dt+\sigma_t dW_t$ to move samples in time such that the marginals are $q_t$. 

Entropic AM objective aims to recover the unique $\tilde{s}_t^*(x)$, as described by the above proposition. In order to learn the diffusion velocity vector, we define the \emph{variational action} $s_t(x)$, parameterized by a neural network, that approximates $\tilde{s}_t^*(x)$, by minimizing the ``$\textsc{e-action-gap}$'' objective
\small
\begin{align*}%
   \textsc{e-action-gap}(s,\tilde{s}^*) \coloneqq \fr 12\int_0^1 \EE_{q_{t}\left(x\right)}{\norm{\g s_{t}\left(x\right)-\g \tilde{s}_{t}^{*}\left(x\right)}^{2}}dt \,.
\end{align*}
\normalsize
Note that while the \textsc{e-action-gap} is similar to the original \textsc{action-gap} objective, it minimizes the distance to $\tilde{s}_t^*$, which is different than $s_t^*$.  As in AM, this objective is intractable since we do not have access to $\g \tilde{s}_t^*$.   However, we derive a tractable objective in the following proposition.
\begin{restatable}{proposition}{eamloss}\label{prop:eam_loss}
For an arbitrary variational action $s$, the $\textsc{e-action-gap}(s,s^*)$ can be decomposed as the sum of an intractable constant $\K$, and a tractable term $\mathcal{L}_{\text{\emph{eAM}}}\left(s\right)$ which can be minimized:
\begin{align}
    \textsc{e-action-gap}(s,\tilde{s}^*) = \mathcal{L}_{\text{\emph{eAM}}}\left(s\right) + \K_\text{\emph{eAM}}\,, \nonumber 
\end{align}
where $\mathcal{L}_{\text{\emph{eAM}}}(s)$ is the entropic Action Matching objective, which we minimize
\begin{align}
    \mathcal{L}_{\text{\emph{eAM}}}&(s)\coloneqq \EE_{q_{{0}}\left(x\right)}\left[s_{{0}}\left(x\right)\right]-\EE_{q_{{1}}\left(x\right)}\left[s_{{1}}\left(x\right)\right]\\
&+\int_{0}^{1} \EE_{q_{t}\left(x\right)}\left[\fr 12\norm{\g s_{t}\left(x\right)}^{2}+\fr{\pa s_{t}}{\pa t}\left(x\right)+\fr{\sigma_t^2}{2}\Delta s_t \right]dt\, \nonumber
\end{align}
\end{restatable}
See \cref{app:eam_proof} for the proof. 
The constant $\K_\text{eAM}$ is the entropic kinetic energy, discussed in \cref{app:eam_proof}.

\paragraph{Connection with Entropic Optimal Transport}
In \cref{app:infinitesimal_eam}, we describe connections between the eAM objective and dynamical formulations of  entropy-regularized optimal transport \citep{cuturi2013sinkhorn} or Schrödinger Bridge \citep{leonard2014survey, chen2016relation, chen2021stochastic} problems. 

\vheader
\subsection{Unbalanced Action Matching}\label{sec:unbalanced_am}

In this section, we further extend the scope of underlying dynamics which can be learned by Action Matching by allowing for the creation and destruction of probability mass via a growth rate $g_t(x)$.   
This term is useful to account for cell growth and death in trajectory inference for single-cell biological dynamics \citep{schiebinger2019optimal,  tong2020trajectorynet, baradat2021regularized, lubeck2022neural, chizat2022trajectory}, and is well-studied in relation to \textit{unbalanced} optimal transport problems \citep{chizat2018interpolating, chizat2018scaling,  chizat2018unbalanced, liero2016optimal,liero2018optimal, kondratyev2016new}.

To introduce \textit{unbalanced} Action Matching (uAM), consider the following ODE, which attaches importance weights to each sample and updates the weights according to a growth rate $g_t(x)$ while transporting the samples in space, 
\begin{align}\label{eq:uam_ode}
    \frac{d}{dt}  x(t) &= v_t(x(t))\,, \quad x(t=0) = x\,.\\
    \frac{d}{dt}  \log w_t(x(t)) &= g_t(x(t))\,, \quad w(t=0) = w\,. \label{eq:iw_ode}
\end{align}
where $v_t$ is the vector field moving particles, similar to continuity equation, $w_t(x)$ is the importance weight of particles, and $g_t(x(t))$ the growth rate of particles.   These importance weights can grow or shrink over time, allowing the particles to be destroyed or create mass probability without needing to transport the particles. The evolution of density governing the importance weighted ODE is given by the following %
continuity equation:
\begin{align}
\pat q_t = - \diver{q_t v_t} + q_t g_t \,. \label{eq:growth_continuity}
\end{align}
In the following proposition, we extend \cref{th:ac_cont_short} to show that any distributional path (under mild conditions), regardless of how it was generated in the state-space, can be modeled with the importance weighted ODE.

\begin{restatable}{proposition}{uamexist}\label{prop:uam_exist}
    Consider a continuous dynamic with density evolution $q_t$ satisfying mild conditions.
    Then, there exists a unique function $\hat{s}_t^*(x)$, called the  ``unbalanced action'', such that velocity field  $v_t^*(x)=\g \hat{s}_t^*(x)$ and growth term $g_t^*(x)= \hat{s}_t^*(x)$ satisfy the importance weighted continuity equation:
    \begin{align}
        \pat q_t = - \diver{q_t \nabla \hat{s}_t^*} + q_t \hat{s}_t^*\,, 
        \label{eq:unbalanced_cont_eq}
    \end{align}
\end{restatable}
See \cref{app:uam_proof} for the proof. This proposition indicates that we can use the importance weighted ODE
\begin{align}\label{eq:uam_ode_solution}
    \frac{d}{dt}  x(t) &= \g \hat{s}_t^*(x(t))\,, \quad x(t=0) = x\,,\\
    \frac{d}{dt}  \log w_t(x(t)) &= \hat{s}_t^*(x(t))\,, \quad w(t=0) = w\,, \label{eq:uam_ode_solution_weights}
\end{align}
to move the particles and update their weights in time, such that the marginals are $q_t$.

Remarkably, the optimal velocity vector field $v_t^* = \g \hat{s}_t^*$ and growth rate $g_t^* = \hat{s}_t^*$ in \cref{prop:uam_exist} are linked to a single action function $\hat{s}_t^*(x)$.   Thus, for learning the variational action $s_t(x)$ in unbalanced AM, we add a term to the ``$\textsc{unbalanced-action-gap}$'' objective which encourages $s_t$ to match $\hat{s}_t^*$,
\small
\begin{align}\label{eq:uam_action_gap}
   \textsc{u-action-gap}(s,\hat{s}^*) \coloneqq \fr 12\int_0^1 \EE_{q_{t}\left(x\right)}{\norm{\g s_{t}\left(x\right)-\g \hat{s}_{t}^{*}\left(x\right)}^{2}}dt  \nonumber \\
   + \fr 12\int_0^1 \EE_{q_{t}\left(x\right)}{\norm{s_{t}\left(x\right)- \hat{s}_{t}^{*}\left(x\right)}^{2}}dt \nonumber  \,.
\end{align}
\normalsize
As before, $\textsc{u-action-gap}(s,\hat{s}^*)$ objective is intractable since we do not have access to $\hat{s}_t^*$. However, as the following proposition shows, 
we can still derive a tractable objective.
\begin{restatable}{proposition}{uamloss}\label{prop:uam_loss}
For an arbitrary variational action $s$, the $\textsc{u-action-gap}(s,\hat{s}^*)$ can be decomposed as the sum of intractable constants 
$\K$ and $\G$, 
and a tractable term $\mathcal{L}_{\text{\emph{uAM}}}\left(s\right)$
\begin{align}
    \hspace*{-.2cm}\textsc{u-}\textsc{action-gap}(s,\hat{s}^*) =  \K_\text{\emph{uAM}}
    +\G_\text{\emph{uAM}}
    + \mathcal{L}_{\text{\emph{uAM}}}\left(s\right) \nonumber
\end{align}
where $\mathcal{L}_{\text{\emph{uAM}}}(s)$ is the unbalanced Action Matching objective, which we minimize
\begin{align}
\mathcal{L}_{\text{\emph{uAM}}}(s)&\coloneqq \EE_{q_{{0}}\left(x\right)}\left[s_{{0}}\left(x\right)\right]-\EE_{q_{{1}}\left(x\right)}\left[s_{{1}}\left(x\right)\right]\\
&+\int_0^1 \EE_{q_{t}\left(x\right)}\left[\fr 12\norm{\g s_{t}\left(x\right)}^{2}+\fr{\pa s_{t}}{\pa t}\left(x\right)+\fr{1}{2} s_t^2 \right]dt\,. \nonumber
\end{align}
\end{restatable}
See \cref{app:uam_proof} for the proof. The constants $\K_{\text{uAM}}$ and $\G_{\text{uAM}}$ are the unbalanced kinetic and growth energy, defined in \cref{app:uam_proof} and \ref{app:infinitesimal_uam}.
We note that the entropic and unbalanced extensions of Action Matching can also be combined, as is common in biological applications \citep{schiebinger2019optimal, chizat2022trajectory}.
To showcase how uAM can handle creation and destruction of mass, without transporting particles, we provide a mixture of Gaussians example in \cref{app:uam}.

\paragraph{Connection with Unbalanced Optimal Transport}
In \cref{app:infinitesimal_uam}, we show that at any given time $t$, the optimal dynamics of uAM along the curve is optimal in the sense of the unbalanced optimal transport \citep{chizat2018interpolating, liero2016optimal, kondratyev2016new} between two infinitesimally close distributions $q_t$ and $q_{t+h}$.

\vheader
\subsection{Action Matching with Convex Costs} \label{sec:ccosts}
In \cref{app:lagr_am}, we further extend AM to minimize kinetic energies defined using an arbitrary strictly convex cost $c(v_t)$ (\citet[Ch. 7]{villani2009optimal}).   For a given path $q_t$, consider 
\small
\begin{align*}
\mathcal{K}_{c\text{AM}} \coloneqq &\inf \limits_{v_t} \int_{0}^{1} \EE_{q_{t}(x)}[c(v_t)] dt
 \quad \text{s.t.} \quad \fr{\pa}{\pa t}q_{t}=-\g\cd\left(q_{t}v_{t}\right).
\nonumber
\end{align*}
\normalsize
In this case, the unique vector field tracing the density evolution of $q_t$ becomes $v_t^*=  \nabla c^*(\nabla \bar{s}^*_t)$, where $c^*$ is the convex conjugate of the $c$.   The corresponding action gap becomes an integral of the Bregman divergence generated by $c^*$,
\begin{align}
\textsc{action-gap}_{c^*}(s_t, \bar{s}_t^*) \coloneqq  \int \limits_0^1  \mathbb{E}_{q_t(x)} \big[ D_{c^*}[ \nabla s_t : \nabla \bar{s}_t^*] \big] dt . \nonumber
\end{align}
In practice, we can minimize $\textsc{action-gap}_{c^*}$ using the following
$c$-Action Matching loss:
\small
\begin{align}
\mathcal{L}_{\text{cAM}}(s_t) \coloneqq \int s_0 (x_0) ~q_0(x_0)dx_0 - \int s_1 (x_1) ~q_1(x_1)dx_1 \nonumber \\ 
   + \int \limits_0^1 \int 
      \left[
      {  c^*(\nabla s_t(x_t))  } + \frac{\partial s_t(x_t)}{\partial t} \right]  ~q_t(x_t)dx_t dt \,.  \nonumber 
\end{align}
\normalsize
For $c(\cdot) = c^*(\cdot) = \frac{1}{2}\| \cdot \|^2$, we recover standard AM. %
Importantly, the continuity equation for this formulation is
\begin{align}
    \pat q_t = - \diver{q_t \g c^*(\g \bar{s}_t^*)}\,.
\end{align}
Thus, the sampling is done by integrating the following vector field
\begin{align}
    \frac{d}{dt} x(t) = \g c^*(\g \bar{s}_t^*, t), \;\;\; x(t=0)=x\,.
\end{align}

\vheader
\vheader
\section{Applications of Action Matching 
\protect\footnote{The code is available at \href{https://github.com/necludov/jam}{github.com/necludov/jam}}}
\label{sec:applications}
In this section, we discuss and empirically study applications of Action Matching.
We first consider the scenario when the samples from the dynamics $q_t$ are given as a dataset, which is the case for applications in biology and physics.
Furthermore, we demonstrate applications of Action Matching in generative modeling, 
where we would like to learn a tractable model of a target distribution represented by a dataset of samples $x_1 \sim q_1$ without given intermediate samples or densities.
\vheader
\subsection{Population Dynamics in Biology}\label{sec:bio}

\begin{table*}[t]
\vspace*{-.1cm}
\centering
\resizebox{0.8\textwidth}{!}{
\begin{tabular}{ccccccc}
\toprule
Distance$\downarrow$ & OT-flow & Trajectory-Net & IPF & Neural SDE & NLSB & eAM (ours)\\
\midrule
$W_2(q_{t_1}, \hat{q}_{t_1})$ & $0.75$ & $0.64$ & $0.65 \pm 0.016$ & $0.62 \pm 0.016$ & $0.63 \pm 0.015$ & $\mathbf{0.58 \pm 0.015}$ \\
$W_2(q_{t_2}, \hat{q}_{t_2})$ & $0.93$ & $0.87$ & $0.78 \pm 0.021$ & $0.78 \pm 0.021$ & $\mathbf{0.75 \pm 0.017}$ & $0.77 \pm 0.016$\\
$W_2(q_{t_3}, \hat{q}_{t_3})$ & $0.93$ & $0.78$ & $0.76 \pm 0.018$ & $0.77 \pm 0.017$ & $0.75 \pm 0.014$ & $\mathbf{0.72 \pm 0.007}$\\
$W_2(q_{t_4}, \hat{q}_{t_4})$ & $0.88$ & $0.89$ & $0.76 \pm 0.017$ & $0.75 \pm 0.017$ & $\mathbf{0.72 \pm 0.010}$ & $0.74 \pm 0.017$\\
\bottomrule
\end{tabular} }
\vspace{-.2cm}
\caption{Performance for embryoid body scRNA-seq data as measured by the $W_2$ distance computed between test data marginals and predicted marginals from the previous test data marginals. Numbers for concurrent methods are taken from \citep{koshizuka2022neuralsb}.}
\label{tab:scRNA}
\vspace*{-.2cm}
\end{table*}

Action Matching is a natural approach to population dynamics inference in biology.
Namely, given the marginal distribution of cells at several timesteps one wants to learn the model of the dynamics of cells \citep{tong2020trajectorynet, bunne2021learning, huguet2022manifold, koshizuka2022neuralsb}.
It is crucial that Action Matching does not rely on the trajectories of cells since the measurement process destroys cells, thus the individual trajectories are unavailable.
We use entropic Action Matching for this task since the ground truth processes are guided by the Brownian motion. 
Moreover, if the diffusion coefficient is known from the experiment conditions and the drift term is a gradient flow, entropic Action Matching recovers the ground truth drift term.
\paragraph{Synthetic Data} We consider synthetic data from \citep{huguet2022manifold}, which simulates natural dynamics arising in cellular differentiation, including branching and merging of cells.
We generate three datasets with $5,10,15$ time steps and learn entropic Action Matching with a fixed diffusion coefficient.
In \cref{fig:petal_results}, we compare our method with MIOFlow \citep{huguet2022manifold} by measuring the distance between the empirical marginals from the dataset and the generated distributions. 
We see that Action Matching outperforms MIOFlow both in terms of Wasserstein-2 distance and Maximum Mean Discrepancy (MMD) \citep{gretton2012kernel}.
Moreover, unlike MIOFlow, Action Matching performance does not degrade with more timesteps or finer granulation of the dynamics in time.
In \cref{fig:petal_trajectories}, we demonstrate that the learned trajectories for Action Matching stay closer to the data marginals compared to the trajectories of MIOFlow.
\vspace{-10pt}
\paragraph{Embryoid scRNA-Seq Data}
For a real data example, we consider a embryoid body single-cell RNA sequencing dataset from \citet{moon2019visualizing}.
We follow the experimental setup from NLSB \citep{koshizuka2022neuralsb} and learn entropic Action Matching where the  sample space is a $5$-dimensional PCA decomposition of the original data.
For Action Matching, we interpolate the data using mixtures between given time steps to obtain data which is more dense in time.
In \cref{tab:scRNA}, we compare our method with OT-flow \citep{onken2021ot}, Trajectory-NET \citep{tong2020trajectorynet}, Iterative Proportional Fitting (IPF) \citep{de2021diffusion}, Neural SDE \citep{li2020scalable}, and Neural Lagrangian Schrodinger Bridge (NLSB) \citep{koshizuka2022neuralsb}. We find that eAM performs on par with NLSB and outperforms other methods. 

\begin{figure}[t]
    \centering
    \includegraphics[width=0.49\textwidth]{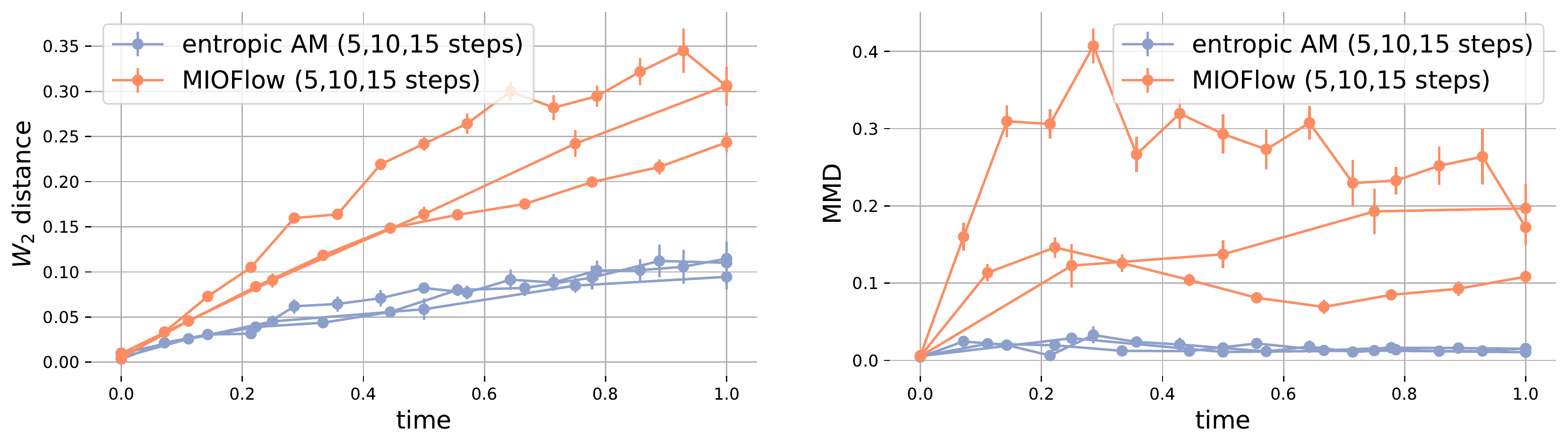}
    \vskip -10pt
    \caption{Performance of entropic Action Matching and MIOFlow on the synthetic data. We simulate the dynamics starting from the initial data distribution and estimate Wasserstein-2 distance and MMD between generated distributions and dataset marginals.}
    \label{fig:petal_results}
\end{figure}%
\begin{figure}[t]
    \centering
    \includegraphics[width=0.49\textwidth]{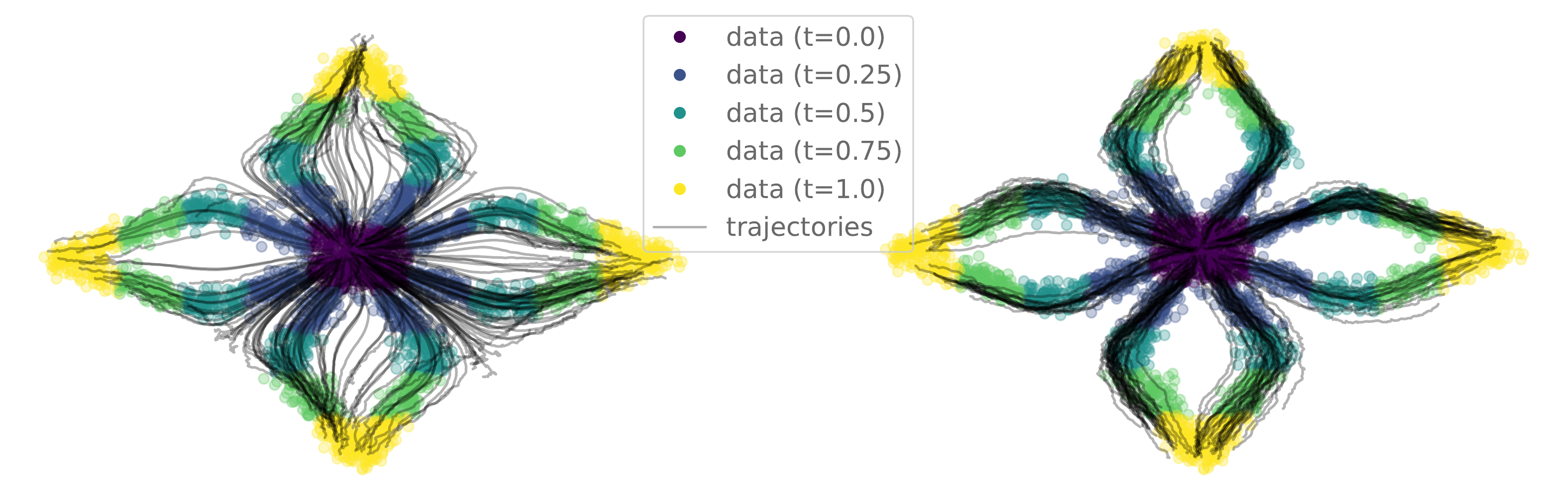}
    \vskip -15pt
    \caption{Generated trajectories by MIOFlow (left) and entropic Action Matching (right). The training data is scattered.}
    \label{fig:petal_trajectories}
    \vspace*{-.2cm}
\end{figure}

\vheader
\subsection{Quantum System Simulation}\label{sec:schrodinger}
In this section, we apply Action Matching to learn the dynamics of a quantum system evolving according to the Schrödinger equation.
The Schrödinger equation describes the evolution of many quantum systems, and in particular, it describes the physics of molecular systems.
Here, for the ground truth dynamics, we take the dynamics of an excited state of the hydrogen atom, which is described by the following equation
\begin{align}\label{eq:schroedinger_hydrogen}
    i\deriv{}{t}\psi(x,t) = -\frac{1}{\norm{x}}\psi(x,t) -\frac{1}{2}\nabla^2\psi(x,t).
\end{align}
The function $\psi(x,t): \mathbb{R}^3\times \mathbb{R} \to \mathbb{C}$ is called a wavefunction and completely describes the state of the quantum system.
In particular, it defines the distribution of the coordinates $x$ by 
the density
$q_t(x) := |\psi(x,t)|^2$, which is defined through the evolution of $\psi(x,t)$ in \cref{eq:schroedinger_hydrogen}.

For the baseline, we take Annealed Langevin Dynamics (ALD) as considered in \citep{song2019generative}.
It approximates the ground truth dynamics using only scores of the distributions by running the approximate MCMC method (which does not have access to the densities) targeting the intermediate distributions of the dynamics (see \cref{alg:ALD}).
For the estimation of scores, we consider Score Matching (SM) \citep{hyvarinen2005estimation}, Sliced Score Matching (SSM) \citep{song2020sliced}, and additionally evaluate the baseline using the ground truth scores.
For further details, we refer the reader to \cref{app:schrodinger} and the supplemented code.

Action Matching outperforms both Score Matching and Sliced Score Matching, precisely simulating the true dynamics (see \cref{tab:schrodinger_results}).
Despite that both SM and SSM accurately recover the ground truth scores for the marginal distributions (see the right plot in \cref{fig:mmd_se} of \cref{app:schrodinger}), one cannot efficiently use them for the sampling from the ground truth dynamics.
Note, that even using the ground truth scores in ALD does not match the performance of Action Matching (see \cref{tab:schrodinger_results}) since it is itself an approximation to the Metropolis-Adjusted Langevin Algorithm.
Finally, we provide animations of the learned dynamics for different methods (see \href{https://github.com/necludov/action-matching}{github.com/necludov/action-matching}) to illustrate the performance difference.

\begin{table}
    \centering
    \resizebox{0.7\columnwidth}{!}{
    \begin{tabular}{lc}
    \toprule
    Method & Average MMD$\downarrow$ \\
    \midrule
    AM (ours) & $\mathbf{5.7\cdot 10^{-4}\pm 3.1\cdot 10^{-4}}$ \\ 
    ALD + SM & $4.8\cdot 10^{-2}\pm 4.8\cdot 10^{-3}$ \\
    ALD + Sliced SM & $4.7\cdot 10^{-2}\pm 4.0\cdot 10^{-3}$ \\
    ALD + True Scores & $3.6\cdot 10^{-2}\pm 4.1\cdot 10^{-4}$\\
    \bottomrule
    \end{tabular}}
    \caption{
    Performance of Action Matching and ALD for the Schrödinger equation simulation. We report average MMD over time between the data and the generated samples. For ALD, we use different estimations of scores: Score Matching (SM), Sliced SM, ground truth scores.}
    \label{tab:schrodinger_results}
    \vspace*{-.5cm}
\end{table}

\vheader
\subsection{Generative Modeling}\label{sec:gen_modeling}
In addition to learning
stochastic dynamics in natural sciences, Action Matching has a wide range of applications in generative modeling. 
The key feature of Action Matching for generative modeling is the flexibility to choose \textit{any} dynamics $q_t$  that starts from the prior distribution $q_0$ and arrives to the given data $q_1$.
After using Action Matching to learn a vector field $\nabla s^\star$ which simulates the appropriate dynamics, 
we can sample and evaluate likelihoods using \cref{eq:sampling} and \cref{eq:lkd_eval}.
We consider a flexible approach for specifying dynamics through interpolants in the sample-space in \cref{eq:interpolants}, and instantiate it for specific problem settings in the following paragraphs.   
We defer analysis of experimental results in each setting to \cref{sec:experiments}.

There are many degrees of freedom in the definition of the density path between a given $q_0$ and $q_1 = \text{dataset}$. 
For instance, we consider interpolating between samples $x_0 \sim q_0$ and $x_1 \sim q_1$ from the prior and data distributions using 
\begin{align}
    x_t = \alpha_t(x_0) + \beta_t(x_1), \;\; x_0 \sim q_0(x), \;\; x_1 \sim q_1(x)\,, \label{eq:interpolants}
\end{align}
where $\alpha_t, \beta_t$ are some continuous transformations of the prior and the data correspondingly.
To respect the endpoint marginals $q_0, q_1$, we select $\alpha_t$ and $\beta_t$ such that $\alpha_0(x_0) = x_0, \beta_0(x_1) = 0$ and $\alpha_1(x_0) = 0, \beta_1(x_1) = x_1$.
The path in density space $q_t$ is implicitly defined via these samples.

\paragraph{Unconditional Generation}
A simple choice of $\alpha_t, \beta_t$ is the linear interpolation between a simple prior and the data, 
\begin{align*}
    x_t = (1-t)\cdot x_0 + t\cdot x_1, \;\; x_0 \sim \Normal(0,1), \;\; x_1 \sim \text{dataset}\,,
\end{align*}
We demonstrate below that these dynamics can be learned with Action Matching, which allows us to perform unconditional image generation and log-likelihood evaluation.

\paragraph{Conditional Generation}
We can also consider 
conditional generation tasks by simulating the dynamics only for the subset of dimensions.
For instance, we can lower the resolution of images in the dataset and learn the dynamics that performs super-resolution using
\begin{align*}
    x_t = \text{mask}\cdot x_1 + (1-\text{mask})\cdot ((1-t)\cdot x_0 + t\cdot x_1),\\
    x_0 \sim \Normal(0,1), \;\; x_1 \sim \text{dataset}\,,
\end{align*}
where \texttt{mask} is a binary-valued vector that defines the subset of dimensions to condition on, e.g. which pixels we generate.
Note, we still can learn the dynamics for all the dimensions of vector $x_t$.
For super-resolution, we keep one pixel in each $2\times2$ block, while the remaining pixels are gradually transformed to the standard normal random variable.

\vspace{-.2cm}
\paragraph{Conditional Coupled Generation}
In some applications, the conditions and the predicted variables might be coupled, or share common dimensions.
We still can perform conditional generation by considering the marginal distribution of conditions as a prior.
To showcase this scenario, we consider image colorization, for which we define the following dynamics
\begin{align*}
    x_t = (1-t) \cdot (10^{-1}\cdot x_0 + \text{gray}(x_1)) + t \cdot x_1,\\
    x_0 \sim \Normal(0,1), \;\; x_1 \sim \text{dataset}\,,
\end{align*}
where the function $\text{gray}(x_1)$ returns the grayscale version of image $x_1$.
We found that minor distortion of the inputs with Normal noise stabilizes the training significantly,
while barely destroying 
the visual information of the image (see \cref{fig:evolutions}).
We also can preserve this information by passing the non-distorted grayscale image to the model in addition to the original input $(x,t)$.

\subsubsection{Empirical Study for Generative Modeling}\label{sec:experiments}
We demonstrate the performance of Action Matching for the aforementioned generative modeling tasks.
For evaluation, we choose the CIFAR-10 dataset of natural images.
Despite the fact that we do not consider image generation to be the main application of Action Matching, we argue that it demonstrates the scalability and applicability of the proposed method for the following reasons: (i) the data is high-dimensional and challenging; (ii) the quality of generated samples is easy to assess; (iii) we use known deep learning architectures, hence, remove this component from the scope of our study.  %

\begin{table}
    \centering
    \resizebox{0.48\textwidth}{!}{
    \begin{tabular}{lcccc}
    \toprule
    Models generating from a Gaussian & BPD$\downarrow$  & FID$\downarrow$ & IS$\uparrow$ & NFE$\downarrow$\\
    \midrule
    VP-SDE \text{\footnotesize (uses extra information) }
    & $3.25$ & $3.71$ & $9.12$ & 199\\
    Flow Matching \text{\footnotesize (uses extra information) }
    & $2.99$ & $6.35$ & $-$ & 142  \\
    Baseline (ALD + SSM) & $-$ & $86.29$ & $5.43$ & 1090\\
    AM - generation (ours) & $3.54$ & $10.04$ & $8.60$ & 132\\
    \midrule
    Models for conditional generation &&&&\\
    \midrule
    Baseline (ALD + SSM) - superres / color & $-$ & n/a & n/a & n/a\\
    AM - superres (ours) & $-$ & $1.44$ & $10.93$ & 166\\
    AM - color (ours) & $-$ & $2.47$ & $9.88$ & 89\\
    \bottomrule
    \end{tabular}
    }
    \vskip -5pt
    \caption{
    Model performance for CIFAR-10. When possible, we report log-likelihood in bits per dimension (BPD). For all tasks, we report FID and IS evaluated for $50$k generated images.}
    \label{tab:generation_results}
    \vspace*{-0.4cm}
\end{table}

\begin{figure*}
    \centering
    \includegraphics[width=0.88\textwidth]{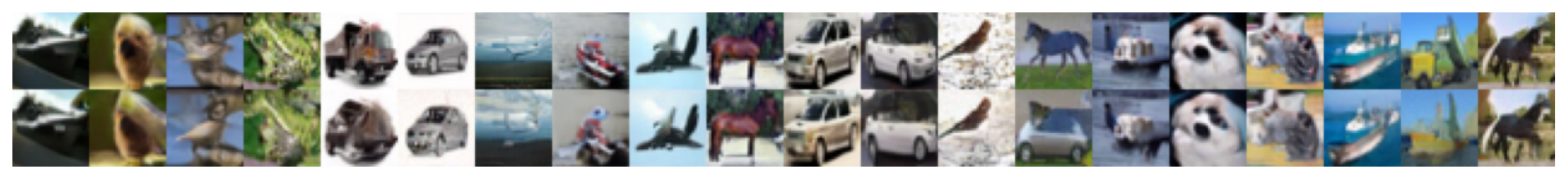}
    \vskip -10pt
    \caption{Generated samples for CIFAR-10 using VP-SDE model (top row) and Action Matching (bottom row).}
    \vskip -10pt
    \label{fig:generation_comparison}
\end{figure*}

For the sake of proper comparison, we consider a baseline model using Annealed Langevin Dynamics (ALD) \citep{song2019generative} and Sliced Score Matching (SSM) \citep{song2020sliced}, which estimates the scores of the marginals $\nabla \log q_t$ \textit{without knowledge} of the underlying dynamics.
As stronger baselines which rely on a known form for the dynamics, we consider the variance-preserving diffusion model (VP-SDE) from \citet{song2020score} \hl{and Flow Matching from \citep{lipman2022flow}.}
We train all models using the same architecture for $500$k iterations and evaluate the negative log-likelihood, FID \citep{heusel2017gans} and IS \citep{salimans2016improved}. \hl{For Flow Matching, we report the numbers from \citep{lipman2022flow}.}

In \cref{tab:generation_results}, we compare the performance of all models, finding that
Action Matching performs much closer to VP-SDE \hl{and Flow Matching (which rely} 
on the knowledge of the dynamics) than to the baseline, \hl{with fewer function evaluations (NFE)}.
Moreover, for the conditional generation, the ALD+SSM baseline fails to learn any meaningful scores.
We argue that this is due to the dynamics starting from a complicated prior distribution, whose scores are difficult to learn without the annealing process to the noise distribution \citep{song2019generative}.
Although VP-SDE achieves better FID and IS scores than AM in \cref{tab:generation_results},  in \cref{fig:generation_comparison}, we find only minor qualitative differences in comparing images generated from the corresponding ODEs starting from the same noise values.
We provide all generations in \cref{app:images} and discuss implementation details in \cref{app:generative_modeling} and \ref{app:implementation}.

\vheader
\section{Related Works}\label{sec:discussion}

A large body of previous work on continuous normalizing flows \citep{chen2018neural} has considered regularization inspired by dynamical optimal transport.  In particular, \citet{finlay2020train, tong2020trajectorynet} regularize the kinetic energy $\frac{1}{2}\mathbb{E}_{q_t}[\|v_t(x)\|^2]$ (and its Jacobian) across time, while \citet{liu2020potential, onken2021ot, koshizuka2022neuralsb} parameterize $v_t$ as the gradient of a potential function $s_t$ and \textit{additionally}  regularize $\mathbb{E}_{q_t}\big[\big|\frac{1}{2}\|\nabla s_t(x)\|^2 + \frac{\partial s_t}{\partial t}\big|\big]$ to be close to zero.   However, these methods all require simulation of and backpropagation through the differential equations during training.  Further, the loss functions for these include maximum likelihood or adversarial losses at the endpoints or entropic OT losses to intermediate distributions.  By contrast, Action Matching is simulation-free and evaluates the action function at the endpoints.
The Action Matching objective in \cref{eq:kilbo} naturally includes minimization of $\mathbb{E}_{q_t}[\frac{1}{2}\|\nabla s_t(x)\|^2 + \frac{\partial s_t}{\partial t}]$, while we show in \cref{app:infinitesimal} that the minimum of $\frac{1}{2}\mathbb{E}_{q_t}[\|v_t(x)\|^2]$ is \textit{also} achieved by the action-minimizing $\nabla s_t(x)$.  In fact, these are dual optimizations, as shown in App. \ref{app:lagr_am} or \citet{mikami2006duality}.  Thus, including \textit{both} regularizers is redundant from the perspective of the Action Matching loss.

Existing methods for trajectory inference often optimize an entropy-regularized (unbalanced) optimal transport loss between the observed samples $q_t$ and samples simulated from a current model $\hat{q}_t$, which requires backpropagation through OT solvers \citep{cuturi2013sinkhorn, chizat2018scaling, cuturi2022optimal}.  In particular, \citet{tong2020trajectorynet, koshizuka2022neuralsb, huguet2022manifold} sample $\hat{q}_t$ using continuous normalizing flows with the regularization described above.  The \textsc{JKOnet} scheme of \citet{bunne2022proximal} seeks to learn a time-independent potential function $\mathcal{F}[q] = \int q(x) s(x)$ for which the observed $q_t$ path is a gradient flow, and generates samples from $\hat{q}_t$ using learned optimal transport maps.   As we discuss in \cref{app:grad_flows}, Action Matching can be viewed as learning a time-dependent potential via $s_t(x)$, without simulation of $\hat{q}_t$ during training, using a simple objective. 

\newcommand{\fmv}{u}
\newcommand{\gtvf}{u(x_t | x_0,x_1)}%

\subsection{Connection with Flow Matching Methods}

Recent related work on Flow Matching (FM) \citep{lipman2022flow, pooladian2023multisample, tong2023conditional}, stochastic interpolants \citep{albergo2022building}, and Rectified Flow \citep{liu2022flow, liu2022rectified} have recently been understood under the umbrella of \textit{bridge matching} (BM) methods \citep{shi2023diffusion, peluchetti2023diffusion, liu2023learning}.   
We highlight several distinctions between Action Matching and flow matching methods in this section.

In particular, consider a deterministic \textit{bridge} which produces intermediate samples $x_t = f_t(x_0, x_1)$ with a tractable conditional vector field $u_t(x_t|x_0,x_1)$.
A natural example is linear interpolation of the endpoint samples.
A marginal vector field $\fmv_t^\theta(x_t)$, which does not condition on a data sample $x_1$ and thus can be used for unconditional generation, is learned to match $u_t(x_t|x_0,x_1)$ using a squared error loss
\small
\begin{align}
   \mathcal{L}_{\text{FM}}(\fmv_t^\theta) = \inf \limits_{\fmv_t} \int_0^1 \mathbb{E}_{q(x_0,x_1)} \big[ \| \gtvf - \fmv_t^\theta(x_t)\|^2 \big] dt . \nonumber
\end{align}
\normalsize
\citet{liu2022rectified} justifies other Bregman divergence losses, which parallels our $c$-Action Matching method in \cref{sec:ccosts}.

The $\mathcal{L}_{\text{FM}}$ objective can be decomposed as in \citep{banerjee2005clustering},
\small
\begin{align}
&\mathbb{E}_{q_{0,1}} \big[ \|\gtvf - \fmv_t^\theta(x_t)\|^2 \big] = \label{eq:pythag} \\
&\mathbb{E}_{q_{0,1}}[ \| \gtvf - \fmv^*_t(x_t)\|^2 ] + \mathbb{E}_{q_t}[ \| \fmv^*_t(x_t) - \fmv_t^\theta(x_t)\|^2 ] \nonumber,
\end{align}
\normalsize
where $\fmv^*_t(x_t)$ is the unique minimizer of $\mathcal{L}_{\text{FM}}$ given by 
\begin{align}
    \fmv^*_t(x_t) = \mathbb{E}_{q(x_0,x_1)}[\gtvf]\,.  \label{eq:expected_vf}
\end{align}
We highlight several key differences between FM and AM.
\paragraph{Target Vector Fields} We can see that the \textsc{action-gap} (\cref{eq:am_action_gap}) in AM is analogous to the second term on the right-hand side of \cref{eq:pythag}, using a different target and parameterization. While both the FM target vector field $\fmv^*_t$ and the AM target vector field $\nabla s_t^*$ yield the same marginals $q_t$ (\cref{th:ac_cont_short}), they have several key differences that can be understood through the Helmholtz decomposition~\citep[Lemma 8.4.2]{ambrosio2008gradient}.
This theorem states that any vector field, such as $\fmv^*_t$, can be uniquely decomposed as $\fmv^*_t=\nabla s_t^* + w_t$, where $\nabla s_t^*$ is the gradient-field component, and $w_t$ is the divergence-free component, $\diver{q_t w_t}=0$.
Among all vector fields respecting marginals $q_t$, the AM target $\nabla s_t^*$ is the unique vector field without a divergence-free
component, moving the particles optimally in the sense of optimal transport. 
By contrast, the FM target $\fmv^*_t$ may contain a divergence-free component, moving the particles in a way that reflects the underlying \textit{path measure} \citep{shi2023diffusion} defined through the endpoint distribution $q(x_0,x_1)$ and the bridge process. Indeed, the gradient-field component of $\fmv^*_t$ is identical to $\nabla s_t^*$, 
and the
divergence-free component of $\fmv^*_t$ does not influence the marginals.
As the result, the optimal $\nabla s_t^*$ has a lower kinetic energy or dynamical cost 
\begin{align}
 \hspace*{-.2cm} \frac{1}{2} \mathbb{E}_{q_t(x_t)}  \|\nabla s_t^*\|^2  \leq  \frac{1}{2} \mathbb{E}_{q_t(x_t)}  \|\fmv_t^*\|^2   \quad \forall~t\,.
\end{align}

\paragraph{Applications} Starting from the \textsc{action-gap}, Action Matching derives a tractable dual objective which removes the need for a known conditional vector field $u_t(x_t|x_0,x_1)$. 
Thus, Action Matching may be applied in settings where linear interpolation is not a reasonable inductive bias, or no bridge vector field $u_t(x_t|x_0,x_1)$ is available.  
In particular, in \cref{sec:applications}, we applied AM to scientific applications where the dynamics of interest are only given via a dataset of samples from the temporal marginals.
\paragraph{Optimization} We note that FM methods are able to optimize the simple $L_2$ loss $\mathbb{E}[ \| \gtvf - \fmv^\theta_t(x_t) \|^2]$ due to the assumption that the target {conditional} vector field $\gtvf$ is available in closed form.   
Furthermore, parametrizing $ \fmv^\theta_t : [0,1] \times \mathbb{R}^d \rightarrow \mathbb{R}^d$ directly using a neural network, rather than parametrizing $s_t : [0,1] \times \mathbb{R}^d \rightarrow \mathbb{R}$ and differentiating to obtain $\nabla s_t \in \mathbb{R}^d$ in AM, provides computational benefits by saving a backward pass.  
Nevertheless, we have previously noted the benefits of the gradient-field parametrization and the generality of Action Matching. %
\vheader
\section{Conclusion}
\label{sec:conclusion}
In this work, we have presented a family of Action Matching methods for learning continuous dynamics from samples along arbitrary absolutely continuous paths on the 2-Wasserstein space. 
We propose the tractable optimization objective that relies \textit{only} on samples from intermediate distributions.
We further derived three extensions of Action Matching: Entropic AM, that can model stochasticity in the state space dynamics; Unbalanced AM, that allows for creation and destruction of probability mass; and $c$-AM, that incorporates convex cost functions on the state space.

The key property of the proposed objective is that it can be efficiently optimized for a wide range of applications.  We demonstrate this empirically in the physical and natural sciences, where snapshots of samples are often given. Further, we demonstrated competitive performance of Action Matching for generative modeling of images, where a prior distribution and density path can be flexibly constructed depending on the task of interest.

\paragraph{Acknowledgement}
The authors thank Viktor Oganesyan for helpful discussions regarding reproducing VP-SDE results.
AM acknowledges support from the Canada CIFAR AI Chairs program.

\bibliography{icml}
\bibliographystyle{icml2023}

\clearpage

\appendix
\onecolumn
\section{Proofs}
\label{app:proof}

\subsection{Action Matching Proofs}
\label{app:am_proof}
\amunique*
\begin{proof}[Proof of existence and uniqueness]\label{app:am_unique_proof}
\hl{The existence of and uniqueness of the solution could be argued by observing that
\begin{align*}
\pat q_{t} & =-\divp{q_{t}\g{s}_{t}^{*}} & \text{in \ensuremath{X}},\\
\left\langle \g{s}_{t}^{*},\n\right\rangle  & =0 & \text{on \ensuremath{\pa X}},
\end{align*}
where $\n$ is the surface normal, is an elliptic PDE with the Neumann boundary condition, and that it is a classical fact that these PDEs have a solution under mild conditions on $q_{t}$ \citep{mikhailov1976partial}. See \citet{ambrosio2008gradient} for a proof in more general settings.}
\end{proof}
\amloss*
\begin{proof}
{\footnotesize{}
\begin{align*}
 & \textsc{action-gap}\left(s_t,s_t^{*}\right)\\
 & =\fr 12\int_{0}^{1}\om_{t}\EE_{q_{t}\left(x\right)}\norm{\gr s_{t}-\gr s_{t}^{*}}^{2}dt\\
 & =\fr 12\int_{0}^{1}\int_{X}\om_{t}q_{t}\left(x\right)\norm{\gr s_{t}-\gr s_{t}^{*}}^{2}dxdt\\
 & \stackrel{}{=}\fr 12\int_{0}^{1}\int_{X}\om_{t}q_{t}\left(x\right)\norm{\gr s_{t}}^{2}dxdt-\int_{0}^{1}\om_{t}\int_{X}q_{t}\left(x\right)\left\langle \gr s_{t}(x),\gr s_{t}^{*}(x)\right\rangle dxdt+\smash{\overbrace{\fr 12\int\EE_{q_{t}\left(x\right)}\norm{\gr s_{t}^{*}}^{2}dt}^{\K_{\textsc{am}}}}\\
 & =\fr 12\int_{0}^{1}\int_{X}\om_{t}q_{t}\left(x\right)\norm{\gr s_{t}}^{2}dxdt-\int_{0}^{1}\om_{t}\int_{X}\left\langle \gr s_{t}(x),q_{t}\left(x\right)\gr s_{t}^{*}(x)\right\rangle dxdt+\K_{\textsc{am}}\\
 & \stackrel{(1)}{=}\fr 12\int_{0}^{1}\int_{X}\om_{t}q_{t}\left(x\right)\norm{\gr s_{t}}^{2}dxdt+\int_{0}^{1}\om_{t}\int_{X}s_{t}(x)\left[\gr\cd\left(q_{t}\left(x\right)\gr s_{t}^{*}(x)\right)\right]dxdt-\int_{0}^{1}\om_{t}\oint_{\pa X}q_{t}\left(x\right)s_{t}(x)\cancelto{0}{\left\langle \gr s_{t}^{*},d\n\right\rangle }dt+\K_{\textsc{am}}\\
 & =\fr 12\int_{0}^{1}\int_{X}\om_{t}q_{t}\left(x\right)\norm{\gr s_{t}}^{2}dxdt-\int_{0}^{1}\left(\int_{X}\om_{t}s_{t}(x)\fr{\pa q_{t}\left(x\right)}{\pa t}dx\right)dt+\K_{\textsc{am}}\\
 & \stackrel{(2)}{=}\int_{0}^{1}\om_{t}\EE_{q_{t}\left(x\right)}\left[\fr 12\norm{\gr s_{t}\left(x\right)}^{2}\right]dt-\left(\eval{\om_{t}\EE_{q_{t}\left(x\right)}\left[s_{t}(x)\right]}_{t={0}}^{t={1}}-\int_{X}\EE_{q_{t}\left(x\right)}\left[s_{t}(x)\fr{d\om_{t}}{dt}+\om_{t}\fr{\pa s_{t}\left(x\right)}{\pa t}\right]dt\right)+\K_{\textsc{am}}\\
 & =\int_{0}^{1}\om_{t}\EE_{q_{t}\left(x\right)}\left[\fr 12\norm{\gr s_{t}\left(x\right)}^{2}+\fr{\pa s_{t}\left(x\right)}{\pa t}+s_{t}(x)\fr{d\log\om_{t}}{dt}\right]dt-\om_{{1}}\EE_{q_{1}\left(x\right)}\left[s_{1}(x)\right]+\om_{{0}}\EE_{q_{0}\left(x\right)}\left[s_{0}(x)\right]+\K_{\textsc{am}}\\
 & =\L_{\textsc{am}}\left(s\right)+\K_{\textsc{am}}
\end{align*}
}where in (1), we have used integration by parts for divergence operator $\int_{X}\left\langle \g g,\vec f\right\rangle dx=\oint_{\pa X}\left\langle \vec fg,d\vec{n}\right\rangle -\int_{X}g\left(\g\cd\vec f\right)dx$ \hl{and that $\eval{\left\langle \gr s_{t}^{*},d\n\right\rangle }_{\pa X}=0$ due to the Neumann boundary condition (see proof of \cref{th:ac_cont_short} above)},
and in (2) we have used integration by parts.%
\end{proof}
For each distributional path $q_t$, the kinetic energy term only depends on the true actions $s_t^*$ and is defined as
\begin{align}
        \K(\g s_t^*)&\coloneqq\fr 12\int_{0}^{{1}}\EE_{q_{t}\left(x\right)}\norm{\g s_{t}^{*}(x)}^{2}dt \,.\label{eq:kinetic_energy}
\end{align}
Thus, minimizing the $\L_{\text{AM}}\left(s\right)$ can be interpreted as maximizing a variational lower bound on kinetic energy.

\hl{
\begin{restatable}[Adapted from \citet{albergo2022building}]{proposition}{albergo} 
Let $\nabla s_t(x)$ be a learned vector field, continuously differentiable in $(t,x)$ and Lipschitz in $x$ uniformly on $[0,1] \times \mathbb{R}^d$ with Lipschitz constant $K$.   Let $\hat{q}_t$ denote the density induced by $\frac{\partial x}{\partial t} =\nabla s_t(x)$ with $x_0 \sim q_0$.  Then, the squared Wasserstein-2 distance between $\hat{q}_t$ and the ground truth $q_t$ (induced by $\nabla s^*_t(x)$) at each time $\tau \in [0,1]$ is bounded by
\begin{align}
    W_2^2(\hat{q}_\tau, q_\tau) \leq e^{(1+2K)\tau} 
    \int_{0}^{\tau} \EE_{q_{t}\left(x\right)}{\norm{\g s_{t}\left(x\right)-\g s_{t}^{*}\left(x\right)}^{2}}dt \nonumber
\end{align}
where the right-hand side includes the action gap in \cref{eq:am_action_gap}.
\end{restatable}
}
\begin{proof}
\newcommand{\pfx}{{x}}%
\hl{
Consider two flows $\varphi_t(\pfx)$ and $\varphi^*_t(\pfx)$ defined by 
\begin{align}
    \deriv{\varphi_t(\pfx)}{t} = \nabla s_t(\varphi_t(\pfx)) \;\; \text{and} \;\; \deriv{\varphi^*_t(\pfx)}{t} = \nabla s^*_t(\varphi^*_t(\pfx))
\end{align}
correspondingly.
Consider the quantity
\begin{align}
    Q_t = \int dx\; q_{0}(\pfx)\norm{\varphi_t(\pfx) - \varphi^*_t(\pfx)}^2 .
\end{align}
Note that $W_2^2(\hat{q}_\tau,q_\tau) \leq Q_\tau$, since $(\varphi_\tau)_{\#}{q_0} = \hat{q}_\tau$ and $(\varphi^*_\tau)_{\#}{q_0} = q_\tau$ have the correct marginals and thus constitute an admissible coupling for the $W_2$ distance.  
The time derivative of $Q_t$ is
\begin{align}
    \deriv{Q_t}{t} =~& 2\int dx\; q_{0}(\pfx)\inner{\varphi_t(\pfx) - \varphi^*_t(\pfx)}{\deriv{\varphi_t(\pfx)}{t} - \deriv{\varphi^*_t(x)}{t}}\\
    =~& 2\int dx\; q_{0}(\pfx)\inner{\varphi_t(\pfx) - \varphi^*_t(\pfx)}{\nabla s_t(\varphi_t(\pfx)) - \nabla s^*_t(\varphi^*_t(\pfx))} \\
    =~& 2\int dx\; q_{0}(\pfx)\inner{\varphi_t(\pfx) - \varphi^*_t(\pfx)}{\nabla s_t(\varphi_t(\pfx)) - \nabla s_t(\varphi^*_t(\pfx))} \\
    ~&~+ 2\int dx\; q_{0}(\pfx)\inner{\varphi_t(\pfx) - \varphi^*_t(\pfx)}{\nabla s_t(\varphi^*_t(\pfx)) - \nabla s^*_t(\varphi^*_t(\pfx))}. \nonumber
\end{align}
\normalsize
By the Lipshitzness of $\nabla s_t(\pfx)$, we bound the first term as
\begin{align}
    2\inner{\varphi_t(\pfx) - \varphi^*_t(\pfx)}{\nabla s_t(\varphi_t(\pfx)) - \nabla s_t(\varphi^*_t(\pfx))} \leq 2K \norm{\varphi_t(\pfx) - \varphi^*_t(\pfx)}^2.
\end{align}
The second term is bounded as follows
\begin{align}
    \norm{\varphi_t(\pfx) - \varphi^*_t(\pfx)}^2-2\inner{\varphi_t(\pfx) - \varphi^*_t(\pfx)}{\nabla s_t(\varphi^*_t(\pfx)) - \nabla s^*_t(\varphi^*_t(\pfx))} + \norm{\nabla s_t(\varphi^*_t(\pfx)) - \nabla s^*_t(\varphi^*_t(\pfx))}^2 \geq 0,\\
    2\inner{\varphi_t(\pfx) - \varphi^*_t(\pfx)}{\nabla s_t(\varphi^*_t(\pfx)) - \nabla s^*_t(\varphi^*_t(\pfx))} \leq \norm{\varphi_t(\pfx) - \varphi^*_t(\pfx)}^2 + \norm{\nabla s_t(\varphi^*_t(\pfx)) - \nabla s^*_t(\varphi^*_t(\pfx))}^2.
\end{align}
Thus,
\begin{align}
    \deriv{Q_t}{t} \leq (1+2K) Q_t + \int dx\; q_0(\pfx)\norm{\nabla s_t(\varphi^*_t(\pfx)) - \nabla s^*_t(\varphi^*_t(\pfx))}^2,
\end{align}
and by Gronwall's lemma, we have
\begin{align}
    Q_\tau \leq Q_0\exp(\tau(1+2K))\int_0^\tau dt\int dx\; q_0(\pfx)\norm{\nabla s_t(\varphi^*_t(\pfx)) - \nabla s^*_t(\varphi^*_t(\pfx))}^2. \label{eq:final_line_albergo}
\end{align}
Clearly, $Q_0 = 0$.   Using \cref{eq:final_line_albergo} and the fact that $W_2^2(\hat{q}_\tau,q_\tau) \leq Q_\tau$ from above, the proposition is proven.
}
\end{proof}

\subsection{Entropic Action Matching Proofs}
\label{app:eam_proof}
\eamexist*
\begin{proof}[Proof of existence and uniqueness]\label{app:eam_unique_proof}
Consider the PDE
\begin{align*}
\pat q_{t} & =-\divp{q_{t}\gr\tilde{s}_{t}^{*}}+\fr{\si_{t}^{2}}2\Delta q_{t} & \text{in \ensuremath{X}}\\
\left\langle \gr\tilde{s}_{t}^{*},\n\right\rangle  & =\fr{\si_{t}^{2}}2\left\langle \gr\log q_{t},\n\right\rangle  & \text{on \ensuremath{\pa X}}
\end{align*}
with the reparametrization $s_{t}^{*}=\tilde{s}_{t}^{*}-\fr{\si_{t}^{2}}2\log q_{t}$, we can write this PDE as 
\begin{align*}
\pat q_{t} & =-\divp{q_{t}\g\tilde{s}_{t}^{*}}+\fr{\si_{t}^{2}}2\Delta q_{t}\\
 & =-\g \cdot\left[q_{t}\left(\g\tilde{s}_{t}^{*}-\fr{\si_{t}^{2}}2\g\log q_{t}\right)\right]\\
 & =-\g\cdot\left[q_{t}\left(\g\left(\tilde{s}_{t}^{*}-\fr{\si_{t}^{2}}2\log q_{t}\right)\right)\right] \\
 & = -\g\cdot\left(q_{t}\g{s}_{t}^{*}\right)
\end{align*}
with the boundary condition of 
\begin{align*}
\left\langle \gr\left(\tilde{s}^{*}-\fr{\si_{t}^{2}}2\log q_{t}\right),\n\right\rangle &=\langle\nabla s^*_t,\n \rangle  = 0
\end{align*}
From \cref{app:am_unique_proof}, we know $s_{t}^{*}$ exists and is unique (up to a constant), thus $\tilde{s}_{t}^{*}$ also exists and is unique (up to a constant).
\end{proof}
\eamloss*
\begin{proof}{\footnotesize{}
\begin{align*}
 & \textsc{e-action-gap}\left(s_t,\tilde{s}^{*}_t\right)\\
 & =\fr 12\int_{0}^{1}\EE_{q_{t}\left(x\right)}\norm{\gr s_t-\gr\tilde{s}_t^{*}}^{2}dt\\
 & =\fr 12\int_{0}^{1}\int_{X}q_{t}\left(x\right)\norm{\gr s_t-\gr\tilde{s}_t^{*}}^{2}dxdt\\
 & \stackrel{}{=}\fr 12\int_{0}^{1}\int_{X}q_{t}\left(x\right)\norm{\gr s_t}^{2}dxdt-\int_{0}^{1}\int_{X}q_{t}\left(x\right)\left\langle \gr s_{t}(x),\gr\tilde{s}_{t}^{*}(x)\right\rangle dxdt+\smash{\overbrace{\fr 12\int\EE_{q_{t}\left(x\right)}\norm{\gr\tilde{s}_{t}^{*}}^{2}dt}^{\K_{\text{eAM}}}}\\
 & =\fr 12\int_{0}^{1}\int_{X}q_{t}\left(x\right)\norm{\gr s_t}^{2}dxdt-\int_{0}^{1}\int_{X}\left\langle \gr s_{t}(x),q_{t}\left(x\right)\gr\tilde{s}_{t}^{*}(x)\right\rangle dxdt+\K_{\text{eAM}}\\
 & \stackrel{(1)}{=}\fr 12\int_{0}^{1}\int_{X}q_{t}\left(x\right)\norm{\gr s_t}^{2}dxdt+\int_{0}^{1}\int_{X}s_{t}(x)\left[\gr\cd\left(q_{t}\left(x\right)\gr\tilde{s}_{t}^{*}(x)\right)\right]dxdt-\fr{\si_{t}^{2}}2\int_{0}^{1}\oint_{\pa X}s_{t}\left\langle \gr q_{t},d\n\right\rangle dt+\K_{\text{eAM}}\\
 & \stackrel{(2)}{=}\fr 12\int_{0}^{1}\int_{X}q_{t}\left(x\right)\norm{\gr s_t}^{2}dxdt-\int_{0}^{1}\left(\int_{X}s_{t}(x)\fr{\pa}{\pa t}q_{t}\left(x\right)dx\right)dt+\fr{\si_{t}^{2}}2\int_{0}^{1}\left(\int_{X}s_{t}(x)\Delta q_{t}dx\right)dt-\fr{\si_{t}^{2}}2\int_{0}^{1}\oint_{\pa X}s_{t}\left\langle \gr q_{t},d\n\right\rangle dt+\K_{\text{eAM}}\\
 & \stackrel{(3)}{=}\int_{0}^{1}\EE_{q_{t}\left(x\right)}\left[\fr 12\norm{\gr s_{t}\left(x\right)}^{2}\right]dt-\left(\eval{\EE_{q_{t}\left(x\right)}\left[s_{t}(x)\right]}_{t={0}}^{t={1}}-\int_{X}\EE_{q_{t}\left(x\right)}\left[\fr{\pa s_{t}\left(x\right)}{\pa t}\right]dt\right)+\fr{\si_{t}^{2}}2\int_{0}^{1}\left(\int_{X}q_{t}(x)\Delta s_{t}dx\right)dt\\
 & \hphantom{=}-\fr{\si_{t}^{2}}2\int_{0}^{1}\oint_{\pa X}q_{t}\left\langle \gr s_{t},d\n\right\rangle dt+\K_{\text{eAM}}\\
 & =\int_{0}^{1}\EE_{q_{t}\left(x\right)}\left[\fr 12\norm{\gr s_{t}\left(x\right)}^{2}+\fr{\pa s_{t}\left(x\right)}{\pa t}+\fr{\si_{t}^{2}}2\Delta s_{t}\right]dt-\EE_{q_{1}\left(x\right)}\left[s_{1}(x)\right]+\EE_{q_{0}\left(x\right)}\left[s_{0}(x)\right]-\fr{\si_{t}^{2}}2\int_{0}^{1}\oint_{\pa X}q_{t}\left\langle \gr s_{t},d\n\right\rangle dt+\K_{\text{eAM}}\\
 & =\mathcal{L}_{\text{eAM}}\left(s\right)+\K_{\text{eAM}}
\end{align*}}
where in (1), we have used the integration by parts for divergence operator $\int_{X}\left\langle \g g,\vec f\right\rangle dx=\oint_{\pa X}\left\langle \vec fg,d\vec{n}\right\rangle -\int_{X}g\left(\g\cd\vec f\right)dx$, \hl{and the Neumann boundary condition from \cref{app:eam_unique_proof}}
{\footnotesize{}
\begin{align*}
\int_{X}\left\langle \gr s_{t}(x),q_{t}\left(x\right)\gr\tilde{s}_{t}^{*}(x)\right\rangle dx & =\oint_{\pa X}q_{t}\left(x\right)s_{t}(x)\left\langle \gr\tilde{s}_{t}^{*},d\n\right\rangle -\int_{X}s_{t}(x)\gr\cd\left(q_{t}\left(x\right)\gr\tilde{s}_{t}^{*}(x)\right)dx & {\color{gray}\text{integration by parts}}\\
 & =\fr{\si_{t}^{2}}2\oint_{\pa X}s_{t}q_{t}\left\langle \gr\log q_{t},d\n\right\rangle -\int_{X}s_{t}(x)\gr\cd\left(q_{t}\left(x\right)\gr\tilde{s}_{t}^{*}(x)\right)dx & {\color{gray}\text{boundary condition}}\\
 & =\fr{\si_{t}^{2}}2\oint_{\pa X}s_{t}\left\langle \gr q_{t},d\n\right\rangle -\int_{X}s_{t}(x)\gr\cd\left(q_{t}\left(x\right)\gr\tilde{s}_{t}^{*}(x)\right)dx
\end{align*}
}

In (2), we have used the Fokker-Planck equation: $\divp{q_{t}\g s_{t}^{*}(x)}=-\pat q_{t}+\fr{\si_{t}^{2}}2\Delta q_{t}$.

To derive (3), we have
{\footnotesize{}\begin{align*}
\int_{X}s_{t}\left(x\right)\Delta q_{t}dx-\int_{\pa X}s_{t}\left\langle \gr q_{t},d\n\right\rangle  & =\int_{X}s_{t}\left(x\right)\left[\divp{\gr q_{t}}\right]dx-\int_{\pa X}s_{t}\left\langle \gr q_{t},d\n\right\rangle \\
 & =-\int_{X}\left\langle \gr s_{t},\gr q_{t}\right\rangle dx & {\color{gray}\text{integration by parts for divergence theorem}}\\
 & =-\int_{X}\left\langle \gr q_{t},\gr s_{t}\right\rangle dx & {\color{gray}\text{by symmetry}}\\
 & =\int_{X}q_{t}\left(x\right)\Delta s_{t}dx-\int_{\pa X}q_{t}\left\langle \gr s_{t},d\n\right\rangle  & {\color{gray}\text{following the reverse of the first 2 steps}}
\end{align*}}
\hl{Finally, note that the term $\fr{\si_{t}^{2}}2\int_{0}^{1}\oint_{\pa X}q_{t}\left\langle \gr s_{t},d\n\right\rangle dt$ can be dropped if $q_{t}$ vanishes on the boundary $\pa X$.} 
\end{proof}
For each distributional path $q_t$, the \emph{entropic kinetic energy} term only depends on the true entropic action $\tilde{s}_t^*$ and is defined as
\begin{align}
        \K_\text{eAM}(\g s_t^*)&\coloneqq\fr 12\int_{0}^{{1}}\EE_{q_{t}\left(x\right)}\norm{\g \tilde{s}_t^*(x)}^{2}dt \,.\label{eq:entropic_kinetic_energy}
\end{align}
Thus, minimizing $\L_{\text{eAM}}\left(s\right)$ can be interpreted as maximizing a variational lower bound on the entropic kinetic energy.

\subsection{Unbalanced Action Matching Proofs}
\label{app:uam_proof}
\uamexist*
\begin{proof}[Proof of existence and uniqueness]
The existence and uniqueness of the solution could be argued by observing that
\begin{align*}
\pat q_{t} & =-\divp{q_{t}\g\hat{s}_{t}^{*}}+\hat{s}_{t}^{*}q_{t} & \text{in \ensuremath{X}}\\
\left\langle \g\hat{s}_{t}^{*},\n\right\rangle  & =0 & \text{on \ensuremath{\pa X}}
\end{align*}
is an elliptic PDE with the Neumann boundary condition, and that it is a classical fact that these PDEs have a solution under mild conditions on $q_{t}$ \citep{mikhailov1976partial}.

\end{proof}
\uamloss*
\begin{proof}
{\footnotesize{}
\begin{align*}
 & \textsc{u-action-gap}\left(s_t,\hat{s}_t^{*}\right)\\
 & =\fr 12\int_{0}^{1}\EE_{q_{t}\left(x\right)}\norm{\gr s_t-\gr\hat{s}_t^{*}}^{2}dt+\fr 12\int_{0}^{1}\EE_{q_{t}\left(x\right)}\norm{s_t-\hat{s}_t^{*}}^{2}dt\\
 & =\fr 12\int_{0}^{1}\int_{X}q_{t}\left(x\right)\norm{\gr s_t-\gr\hat{s}_t^{*}}^{2}dxdt+\fr 12\int_{0}^{1}\int_{X}q_{t}\left(x\right)\norm{s_t-\hat{s}_t^{*}}^{2}dxdt\\
 & =\int_{0}^{1}\int_{X}q_{t}\left(x\right)\left[\fr 12\norm{\gr s_t}^{2}+\fr 12s_t^{2}\right]dxdt-\int_{0}^{1}\int_{X}q_{t}\left(x\right)\left\langle \gr s_{t}(x),\gr\hat{s}_{t}^{*}(x)\right\rangle dxdt-\int_{0}^{1}\int_{X}q_{t}\left(x\right)s_{t}(x)\hat{s}_{t}^{*}(x)dxdt\\
 & \hphantom{=}+\overbrace{\fr 12\int\EE_{q_{t}\left(x\right)}\norm{\gr\hat{s}_{t}^{*2}}^{2}dt}^{\K_{\text{uAM}}}+\overbrace{\fr 12\int\EE_{q_{t}\left(x\right)}\hat{s}_{t}^{*2}dt}^{\G_{\text{uAM}}}\\
 & =\int_{0}^{1}\int_{X}q_{t}\left(x\right)\left[\fr 12\norm{\gr s_{t}}^{2}+\fr 12s_t^{2}\right]dxdt-\int_{0}^{1}\int_{X}\left\langle \gr s_{t}(x),q_{t}\left(x\right)\gr\hat{s}_{t}^{*}(x)\right\rangle dxdt-\int_{0}^{1}\int_{X}q_{t}\left(x\right)s_{t}(x)\hat{s}_{t}^{*}(x)dxdt+\K_{\text{uAM}}+\G_{\text{uAM}}\\
 & \stackrel{(1)}{=}\int_{0}^{1}\int_{X}q_{t}\left(x\right)\left[\fr 12\norm{\gr s_{t}}^{2}+\fr 12s_t^{2}\right]dxdt+\int_{0}^{1}\int_{X}s_{t}(x)\left[\gr\cd\left(q_{t}\left(x\right)\gr\hat{s}_{t}^{*}(x)\right)-q_{t}\left(x\right)\hat{s}_{t}^{*}(x)\right]dxdt-\int_{0}^{1}\oint_{\pa X}q_{t}\left(x\right)s_{t}(x)\cancelto{0}{\left\langle \gr s_{t}^{*},d\n\right\rangle dt}\\
 & \hphantom{=}+\K_{\text{uAM}}+\G_{\text{uAM}}\\
 & \stackrel{(2)}{=}\int_{0}^{1}\int_{X}q_{t}\left(x\right)\left[\fr 12\norm{\gr s_{t}}^{2}+\fr 12s_t^{2}\right]dxdt-\int_{0}^{1}\left(\int_{X}s_{t}(x)\fr{\pa}{\pa t}q_{t}\left(x\right)dx\right)dt+\K_{\text{uAM}}+\G_{\text{uAM}}\\
 & \stackrel{(3)}{=}\int_{0}^{1}\EE_{q_{t}\left(x\right)}\left[\fr 12\norm{\gr s_{t}\left(x\right)}^{2}+\fr 12s_t^{2}\right]dt-\left(\eval{\EE_{q_{t}\left(x\right)}\left[s_{t}(x)\right]}_{t={0}}^{t={1}}-\int_{X}\EE_{q_{t}\left(x\right)}\left[\fr{\pa s_{t}\left(x\right)}{\pa t}\right]dt\right)+\K_{\text{uAM}}+\G_{\text{uAM}}\\
 & =\int_{0}^{1}\EE_{q_{t}\left(x\right)}\left[\fr 12\norm{\gr s_{t}\left(x\right)}^{2}+\fr{\pa s_{t}\left(x\right)}{\pa t}+\fr 12s_t^{2}\right]dt-\EE_{q_{1}\left(x\right)}\left[s_{1}(x)\right]+\EE_{q_{0}\left(x\right)}\left[s_{0}(x)\right]+\K_{\text{uAM}}+\G_{\text{uAM}}\\
 & =\L_{\text{uAM}}\left(s\right)+\K_{\text{uAM}}+\G_{\text{uAM}}
\end{align*}}
where in (1), we have used integration by parts for divergence operator $\int_{X}\left\langle \g g,\vec f\right\rangle dx=\oint_{\pa X}\left\langle \vec fg,d\vec{n}\right\rangle -\int_{X}g\left(\g\cd\vec f\right)dx$ and \hl{that $\eval{\left\langle \gr s_{t}^{*},d\n\right\rangle }_{\pa X}=0$ due to the Neumann boundary condition (see \cref{app:uam_proof})},
in (2) we have used: $\pat q_{t}=-\divp{q_{t}\g s^{*}}+s^{*}q_{t}$,
and in (3) we have integration by parts.%
\end{proof}
For each distributional path $q_t$, the \emph{unbalanced kinetic energy} and \emph{unbalanced growth energy} terms only depend on the true unbalanced action $\hat{s}_t^*$ and are defined as
\begin{align}
    \K_\text{uAM}(\g \hat{s}_t^*)&\coloneqq\fr 12\int_{0}^{{1}}\EE_{q_{t}\left(x\right)}\norm{\g \hat{s}_t^*(x)}^{2}dt \,. \qquad\qquad
    \G_\text{uAM}(\hat{s}_t^*)\coloneqq\fr 12\int_{0}^{{1}}\EE_{q_{t}\left(x\right)}{ \hat{s}_t^*(x)}^{2}dt \,.\label{eq:kinetic_energy_uam}        
\end{align}
Thus, minimizing the $\L_{\text{uAM}}\left(s\right)$ can be interpreted as maximizing a variational lower bound on the summation of unbalanced kinetic energy and growth energy.

\section{Action Matching and Optimal Transport}\label{app:am_ot}
In this section, we describe various connections between Action Matching and optimal transport.   First, we describe how Action Matching (\cref{app:infinitesimal}), entropic AM (\cref{app:infinitesimal_eam}), and unbalanced AM (\cref{app:infinitesimal_uam}) can be understood as solving `local' versions of optimal transport problems between infinitesimally close distributions, where connections with dynamical OT formulations play a key role \citep{benamou2000computational, chen2016relation, chizat2018interpolating, liero2016optimal}.    In \cref{app:lagr_am}, we generalize AM by considering Lagrangian action cost functions beyond the squared Euclidean OT cost and standard kinetic energy underlying the other results in this paper.   Finally, in \cref{app:grad_flows}, we interpret AM as learning a linear potential functional whose Wasserstein-2 gradient flow traces the given density path $q_t$.

\subsection{Action Matching as Infinitesimal Optimal Transport}\label{app:infinitesimal}

Using \citet{ambrosio2008gradient} Prop. 8.4.6 (see also  \citet{villani2009optimal} Remark 13.10), we can interpret 
$\nabla s_t^*$ in action matching as learning the optimal transport map between two infinitesimally close distributions on the given curve $q_t$, under the squared Euclidean cost.  In particular, by \citet{ambrosio2008gradient} Prop. 8.4.6, we have
\begin{align}
\nabla s_t^* = \lim_{dt \to 0} \frac{1}{dt}(T^*(q_t, q_{t+dt})-\id).  \label{eq:inf_ot}
\end{align}
where $\id$ is the identity mapping and $T^*(q_t, q_{t+dt}) = x_t + \nabla \varphi^*_{t,dt}(x_t)$ is the unique transport map (\citet{gangbo1996geometry} Thm. 4.5) solving the Monge formulation of the Wasserstein-2 distance between neighboring densities $q_t, q_{t+dt}$ 
\small
\begin{align}
    \hspace*{-.2cm} W_2^2(q_t, q_{t+dt}) = \inf_T\bigg\{\int  \|T(x) - x\|^2 q_t(x) dx \; \bigg|\; T_{\#}q_t = q_{t+dt}\bigg\} =  \fr 12 \mathbb{E}_{q_t(x)}\left[\|\nabla \varphi^*_{t,dt}(x)\|^2 \right] \; \,\,\text{s.t.} \,\, \; (x+\nabla \varphi^*_{t,dt})_{\#}q_t = q_{t+dt} \label{eq:w2_monge}
\end{align}
\normalsize
and $T_{\#}q_t$ is the pushforward density of $q_t$ under the map $T$.  Note that the continuity equation suggests the pushforward map $x+ dt \nabla s_t^*(x)$ for small $dt$.
Thus, among all vector fields, the Action Matching objective finds the one that satisfies the continuity equation while minimizing the infinitesimal displacement of samples according to the squared Euclidean cost.   We can also compare \cref{eq:w2_monge} to \cref{eq:s_minimization_w2} below.

 \paragraph{Kinetic Energy and Dynamical Optimal Transport}
To further understand this result, we observe that Action Matching is closely related to the dynamical formulation of optimal transport due to \citet{benamou2000computational}.   
The squared Wasserstein-2 distance between given densities $p_0$ and $p_1$ with densities $q_{0}$ and $q_{1}$ can be expressed as, 
\begin{equation}
\begin{aligned}
W_2^2(\mu, \nu) &= \inf \limits_{q_t} \inf \limits_{v_t} \fr 12\int_{0}^{1} \EE_{q_{t}(x)}\norm{v_{t}(x)}^{2} dt \qquad  
\text{s.t. } \, \, \fr{\pa}{\pa t}q_{t}=-\g\cd\left(q_{t}v_{t}\right), \,\, q_{0} = p_{0} \text{, and }  q_{1} = p_1.
\end{aligned}
\label{eq:benamou_brenier}
\end{equation}
However, in Action Matching, the intermediate densities or distributional path $q_t$ are \textit{fixed} 
via the given samples.    In this case, we can interpret AM as learning a \textit{local} or infinitesimal optimal transport map as in \cref{eq:inf_ot}, although the given $q_t$ may not trace the \textit{global} optimal transport path between $q_0$ and $q_1$.  Nevertheless, 
it can be shown that (\citet{ambrosio2008gradient})
\begin{align}
    \g s_{t}^{*}= \argmin \limits_{v_t}~ \fr 12\int_{0}^{1} \EE_{q_{t}(x)}\left[\norm{v_{t}(x)}^{2}\right] dt \qquad  
\text{s.t. } \, \, \fr{\pa}{\pa t}q_{t}&=-\g\cd\left(q_{t}v_{t}\right). \label{eq:s_minimization_w2}
\end{align}
\normalsize
To show this directly, we introduce a Lagrange multiplier $s_t$ to enforce the continuity equation and integrate by parts with respect to $x$ and $t$ (also see \cref{eq:lagr_derivation_block}),
\small
\begin{align}
\hspace*{-.2cm}
\mathcal{L}(v_t, s_t) 
 &=\sup \limits_{s_t} \inf \limits_{v_t}  \fr 12 \int \limits_{0}^1  \int  \EE_{q_{t}(x)}\left[\norm{v_{t}(x)}^{2}\right] dt  + 
   \int \limits_0^1  \int   s_t(x_t) \Big( \frac{\partial q_t}{\partial t}(x_t) + 
   \nabla \cdot \big( q_t(x_t)  v_t(x_t) \big) \Big)
   dx_t dt 
   \\[2ex]
   &= 
\sup \limits_{s_t} \inf \limits_{v_t} ~  \fr 12\int_{0}^{1} \EE_{q_{t}(x)}\left[\norm{v_{t}(x)}^{2}\right] dt +   \mathbb{E}_{q_t(x_t)}[s_t(x_t)] \Big|^{t=1}_{t=0} - \int\limits_{0}^{1} \int  \EE_{q_{t}(x_t)} \left[\frac{\partial s_t(x_t)}{\partial t} +\langle v_t(x_t), \nabla s_t(x_t) \rangle \right]dt  
    \label{eq:ot_lagr}  
\end{align}
\normalsize
Eliminating the optimization with respect to $v_t$, 
we obtain the optimality condition
\begin{align}
v_t(x) = \nabla s_t(x) \label{eq:v_gf}
\end{align}
Plugging this condition into the Lagrangian in \cref{eq:ot_lagr}, we obtain the AM objective from \cref{prop:am_loss}, $\inf \limits_{s_t} \mathcal{L}_{\textsc{AM}}(s_t)$, where $\mathcal{L}_{\textsc{AM}}(s_t) = \int_{0}^{1} \EE_{q_{t}(x)}\big[ \fr12 \norm{\nabla s_{t}(x)}^{2} + \frac{\partial s_t(x_t)}{\partial t} \big] dt -   \mathbb{E}_{q_t(x_t)}[s_t(x_t)] \big|^{t=1}_{t=0}$ .
This objective is an upper bound on the optimal kinetic energy $\mathcal{K}(\g s_t^*)$, which is evaluated using the optimal vector field $v_t^* = \nabla s_t^*$ in \cref{eq:s_minimization_w2}.

\paragraph{Metric Derivative} Finally, the AM objective and \cref{eq:s_minimization_w2} are related to the \textit{metric derivative} in the 2-Wasserstein space $\mathcal{P}_2(\mathcal{X})$.    By \citet{ambrosio2008gradient} Thm. 8.3.1, for 
$q_t, q_{t+dt}$
along an absolutely continuous density path with $\partial_t q_t= -\nabla \cdot (q_t v_t^*)$, we have
\begin{align}
    |q_t^\prime|^2 = \left( \lim \limits_{dt \rightarrow 0} \frac{W_2(q_{t}, q_{t+dt})}{dt}\right)^2 = \mathbb{E}_{q_t(x)}\left[ \|v_t^*(x)\|^2 \right] = \mathbb{E}_{q_t(x)}\left[ \|\nabla s_t^*(x)\|^2 \right] .
\end{align}
The optimal value of the AM objective at each $t$ thus reflects the infinitesimal squared W2 distance along the given path.

\subsection{Action Matching with Lagrangian Costs}\label{app:lagr_am}
Starting from the Benamou-Brenier dynamical formulation of optimal transport in \cref{eq:benamou_brenier}, we can also formulate action matching with more general Lagrangian action cost functions. 
For a Lagrangian $\lagr(x_t, v_t, t)$ which is strictly convex in the velocity $v_t$, we define the dynamical optimal transport problem %
\begin{align}
   \mathcal{T}(p_0, p_1)
   &= \inf \limits_{q_t} \inf \limits_{v_t} \int \limits_{0}^1  \int  \lagr(x_t, v_t, t) ~q_t(x_t) dx_t  dt \quad\,\, \text{subj:} \quad \frac{\partial q_t}{\partial t}(x_t) = \nabla \cdot \big( q_t(x_t) v_t(x_t) \big) 
   \quad q_0 = p_0 \text{  and  } q_1 = p_1
   .  \label{eq:action_min_def} 
   \end{align}
   See \citet{villani2009optimal} Ch. 7 for in-depth analysis of \cref{eq:action_min_def} as an optimal transport problem, where $\lagr(x_t, v_t, t) = \frac{1}{2}\| v_t \|^2$ recovers the Wasserstein-2 distance.
   For action matching, assume that the intermediate path of $q_t$ is given via its samples.   We now seek to learn a vector field which minimizes the Lagrangian while satifying the continuity equation,
   \begin{align}
   \mathcal{K}_{\ell{\text{AM}}}^* \coloneqq \inf \limits_{v_t} \int \limits_{0}^1  \int \lagr(x_t, v_t, t)~ q_t(x_t) dx_t dt \quad\,\, \text{subj:} \quad \frac{\partial q_t}{\partial t} = - 
   \nabla \cdot \big( q_t v_t \big) %
   . \label{eq:lagr_kinetic}
   \end{align}
   For this Lagrangian cost, we prove the following analogue of \cref{prop:am_loss}, which is recovered for $\lagr(x_t, v_t, t) = \frac{1}{2}\| v_t \|^2$.  We first state a definition.
   \newcommand{\hamilarg}{{a_t}}
   \begin{definition}\label{def:hamil} The convex conjugate of a Lagrangian $\lagr(x_t, v_t, t)$ which is strictly convex in $v_t$ is defined as
   \begin{align}
    \hamil(x_t, \hamilarg, t) &= \sup_{v_t} ~\langle v_t, \hamilarg \rangle - \lagr(x_t, v_t, t).
   \end{align}
   Solving for the optimizing argument yields the dual correspondence $v_t = \nabla_\hamilarg \hamil(x_t, \hamilarg, t)$ and $\hamilarg = \nabla_{v_t} \lagr(x_t, v_t, t)$.
   \end{definition}
\begin{proposition} For a Lagrangian $\lagr(x_t, v_t, t)$ which is strictly convex in the velocity $v_t$, the optimization defining $   \mathcal{K}_{\ell\emph{AM}}^*$ in \cref{eq:lagr_kinetic} can be expressed via the dual optimization, 
\begin{align}
\mathcal{K}_{\ell\emph{AM}}^* &= \inf \limits_{s_t}~ \mathcal{L}_{\ell\emph{AM}}(s_t), \qquad \label{eq:kilbo_lagr2} \\
\text{where} \quad\mathcal{L}_{\emph{$\ell$AM}}(s_t) &\coloneqq \int s_0 (x_0) ~q_0(x_0)dx_0 - \int s_1 (x_1) ~q_1(x_1)dx_1 
   + \int \limits_0^1 \int 
      \left[
      {  \hamil(x_t, \nabla s_t(x_t), t)  } + \frac{\partial s_t(x_t)}{\partial t} \right]  ~q_t(x_t)dx_t dt   \nonumber %
\end{align}
For a variational $s_t$, the corresponding action gap between $s_t$ and the optimal $\bar{s}_t^*$, 
is written using a Bregman divergence %
\begin{align}
\mathcal{K}_{\ell\emph{AM}}^* + \mathcal{L}_{\ell\emph{AM}}(s_t) ~ &= ~ \textsc{action-gap}_\hamil(s_t, \bar{s}_t^*) \\
\text{where} \quad &
\textsc{action-gap}_\hamil(s_t, \bar{s}_t^*) \coloneqq  \int \limits_0^1 \int  D_{H}[ \nabla s_t(x_t) : \nabla \bar{s}_t^*(x_t)] ~q_t(x_t) ~dx_t dt 
\end{align}
where the Bregman divergence $D_{H}$ is defined as
\begin{align}
 D_{H}[ \nabla s_t(x_t) : \nabla \bar{s}_t^*(x_t)] = \hamil(x_t, \nabla s_t, t) - \hamil(x_t, \nabla \bar{s}_t^*, t) - \langle v_t^*(x_t), \nabla s_t (x_t)-  \nabla \bar{s}_t^*(x_t) \rangle. \label{eq:breg}
\end{align}
Using \cref{def:hamil}
and the fact that $v_t^*$ and $\nabla \bar{s}_t^*$ are duals related by $v_t^* = \nabla_{\hamilarg} H(x_t, \nabla \bar{s}^*_t, t)$, the Bregman divergence may also be written in mixed parameterization $D_{H}[ \nabla s_t : \nabla \bar{s}_t^*] = D_{L,H}[ v_t^* : \nabla s_t]$,  with
\begin{align}
D_{L,H}[ v_t^* : \nabla s_t(x_t)] = \lagr(x_t, v_t^*, t) + \hamil(x_t, \nabla s_t, t) - \langle v_t^*(x_t), \nabla s_t(x_t) \rangle. \label{eq:breg2}
\end{align}
\end{proposition}
For the case of $\lagr(x_t, v_t, t) = \frac{1}{2}\| v_t \|^2$, the Hamiltonian is $\hamil(x_t, \nabla s_t, t) = \frac{1}{2}\| \nabla s_t \|^2$ and the two parameterizations are self-dual with $v_t = \nabla s_t$.     Using this transformation, the Bregman divergence is simply half the squared Euclidean norm, 
\begin{align}
D_{L,H}[v_t^* : \nabla s_t] = \frac{1}{2}\| v^*_t \|^2 + \frac{1}{2}\| \nabla s_t \|^2 - \langle  v^*_t, \nabla s_t \rangle =\frac{1}{2}\| \nabla \bar{s}_t^* \|^2 + \frac{1}{2}\| \nabla s_t \|^2 - \langle \nabla \bar{s}_t^*, \nabla s_t \rangle 
=  \frac{1}{2} \|\nabla \bar{s}_t^*  - \nabla s_t \|^2.
\end{align}
From \cref{def:hamil}, note that in general, the optimality condition translating between $v_t$ and $\nabla s_t$ is $v_t = \nabla_{\hamilarg} \hamil (x_t, \nabla s_t, t)$ or $\nabla s_t = \nabla_{v_t} \lagr (x_t, v_t, t)$.
   \begin{proof}
   Introducing a Lagrange multipler $s_t(x)$ to enforce the continuity equation constraint, we integrate by parts in both $x$ and $t$ using the assumption of the boundary condition $\langle v_t, \vec{n}\rangle = 0$,
   \small
   \begin{equation}
   \begin{aligned}
  \mathcal{K}_{\ell\text{AM}}^* &= \inf \limits_{v_t} \int \limits_{0}^1  \int \lagr(x_t, v_t, t)~ q_t(x_t) dx_t dt \quad\,\, \text{subj:} \quad \frac{\partial q_t}{\partial t}(x_t) = -\nabla \cdot \big( q_t(x_t)  v_t(x_t) \big) \\
  &=\sup \limits_{s_t} \inf \limits_{v_t}  \int \limits_{0}^1  \int  \lagr(x_t, v_t, t) q_t(x_t) dx_t  dt + 
   \int \limits_0^1  \int   s_t(x_t) \Big( \frac{\partial q_t}{\partial t}(x_t) + 
   \nabla \cdot \big( q_t(x_t)  v_t(x_t) \big) \Big)
   dx_t dt 
   \\[2ex]
   &=  \sup \limits_{s_t} \textcolor{azure(colorwheel)}{\inf \limits_{v_t}}\int \limits_{0}^1  \int  \textcolor{azure(colorwheel)}{\lagr(x_t, v_t, t)} ~ q_t(x_t) dx_t  dt +  \int  s_t(x_t) q_t(x_t) dx_t\Big|^{t=1}_{t=0} - \int\limits_{0}^1 \int  \frac{\partial s_t(x_t)}{\partial t} q_t(x_t) dx_t dt  \\
    &\phantom{=\inf\sup} \qquad 
   \int\limits_{0}^1 \oint_{\pa X} q_t(x_t) s_t(x_t) ~\cancelto{0}{\left \langle v_t(x_t), d\vec{n} \right\rangle} dt
    - \int\limits_{0}^1 \int 
      \textcolor{azure(colorwheel)}{\langle v_t(x_t), \nabla s_t(x_t) \rangle} ~ q_t(x_t) dx_t dt \\
      &{=} \sup \limits_{s_t} \int \limits_0^1 \int   -\Big( 
      \textcolor{azure(colorwheel)}{ \sup\limits_{v_t}\, \, \langle v_t(x_t), \nabla s_t(x_t) \rangle - \lagr(x_t, v_t, t)  } \Big) q_t(x_t) dx_t dt -\int \limits_0^1 \int  \frac{\partial s_t(x_t)}{\partial t} q_t(x_t) dx_t dt +  \int  s_t(x_t) q_t(x_t) dx_t\Big|^{t=1}_{t=0} \\
\end{aligned} \label{eq:lagr_derivation_block}
\end{equation}
\normalsize
We can recognize the highlighted terms as a Legendre transform, where we define the Hamiltonian $\hamil$ as the convex conjugate $\hamil(x_t, \nabla s_t(x_t), t) = \sup\limits_{v_t}\, \, \langle v_t(x_t), \nabla s_t(x_t) \rangle - \lagr(x_t, v_t, t)$ of the Lagrangian $\lagr$.   This finally results in 
\begin{equation}
\begin{aligned}
  \mathcal{K}_{\ell\text{AM}}^* &\coloneqq  \sup \limits_{s_t} \int s_1 (x_1) ~q_1(x_1)dx_1 -\int s_0 (x_0) ~q_0(x_0)dx_0 -
    \int \limits_0^1 \int 
      \Big(
      \textcolor{azure(colorwheel)}{  \hamil(x_t, \nabla s_t(x_t), t)  } + \frac{\partial s_t(x_t)}{\partial t} \Big)  ~q_t(x_t)dx_t dt   .
      \label{eq:action_min}\\
      &= -\mathcal{L}_{\text{$\ell$AM}}(s_t^*)
\end{aligned}
\end{equation} 
where \cref{eq:action_min} can be used to define a corresponding action matching objective $\mathcal{L}_{\text{$\ell$AM}}(s_t)$ for a variational $s_t$ (see \cref{eq:kilbo_lagr2}).   

To calculate the action gap, we subtract $\mathcal{L}_{\text{$\ell$AM}}(s_t)$ from the optimal value of the `kinetic energy' in \cref{eq:lagr_kinetic}, using $v_t^*$
\small
\begin{align*}
  \textsc{action-gap}&_\hamil(s_t, \bar{s}_t^*) \\
  &\coloneqq  \mathcal{K}_{\ell\text{AM}}^* + \mathcal{L}_{\text{$\ell$AM}}(s_t) \\[1.25ex]
  &= \mathcal{L}_{\text{$\ell$AM}}(s_t) - \mathcal{L}_{\text{$\ell$AM}}( \bar{s}_t^*) \qquad \text{(using \cref{eq:action_min})}\\
  &=\int s_0 (x_0) ~q_0(x_0)dx_0 - \int s_1 (x_1) ~q_1(x_1)dx_1 +
    \int \limits_0^1 \int 
      \Big(
      {  \hamil(x_t, \nabla s_t(x_t), t)  } + \frac{\partial s_t(x_t)}{\partial t} \Big)  ~q_t(x_t)dx_t dt \\
      &\phantom{===} - \left(
      \int \bar{s}_0^* (x_0) ~q_0(x_0)dx_0 - \int \bar{s}^*_1 (x_1) ~q_1(x_1)dx_1 +
    \int \limits_0^1 \int 
      \Big(
      {  \hamil(x_t, \nabla \bar{s}_t^*(x_t), t)  } + \frac{\partial \bar{s}_t^*(x_t)}{\partial t} \Big)  ~q_t(x_t)dx_t dt \right) \\
      &\overset{(1)}{=}  \int \limits_0^1 \int 
      \Big(
      {   \hamil(x_t, \nabla s_t(x_t), t) - \hamil(x_t, \nabla \bar{s}_t^*(x_t), t)  }   \Big) ~q_t(x_t)dx_t dt + \int \limits_0^1 \int \frac{\partial q_t(x_t)}{\partial t} \big(  -s_t(x_t)+\bar{s}_t^*(x_t) \big) dx_t dt \\
      &\overset{(2)}{=}  \int \limits_0^1 \int 
      \Big(
      { \hamil(x_t, \nabla s_t(x_t), t) - \hamil(x_t, \nabla \bar{s}_t^*(x_t), t)  }   \Big) ~q_t(x_t)dx_t dt - \int \limits_0^1 \int \left( -\nabla \cdot \big( q_t v_t^*(x_t) \big) \right) \cdot \big( s_t(x_t) - \bar{s}_t^*(x_t) \big) dx_t dt \\
      &\overset{(3)}{=} \int \limits_0^1 \int 
      \Big(
       \hamil(x_t, \nabla s_t(x_t), t) - \hamil(x_t, \nabla \bar{s}_t^*(x_t), t) - \langle v_t^*(x_t) , \nabla s_t(x_t) - \nabla \bar{s}_t^*(x_t) \rangle \Big) ~q_t(x_t)dx_t dt \\
       &=  \int \limits_0^1 \int  D_{H}[ \nabla s_t(x_t) : \nabla \bar{s}_t^*(x_t)]~q_t(x_t)dx_t dt
\end{align*}
\normalsize
where in $(1)$ we use the fact that $\int_0^1 \int -\frac{\partial q_t} {\partial t} s_t dx_t dt = \int s_0 ~q_0 dx_0 - \int s_1 q_1 dx_1 + \int_0^1 \int \frac{\partial s_t}{\partial t} ~q_t dx_t dt$ by integration by parts.   In $(2)$, we use the fact that $v_t^*$ satisfies the continuity equation for $q_t$ from \cref{eq:lagr_kinetic}.    In $(3)$, we integrate by parts with respect to $x$ and recognize the resulting expression as the definition of the Bregman divergence from \cref{eq:breg}.  
\end{proof}

\subsection{Entropic Action Matching and Entropy-Regularized Optimal Transport} \label{app:infinitesimal_eam}

Consider the dynamical formulation of entropy-regularized optimal transport \citep{leonard2014survey, chen2016relation, chen2021stochastic}, which involves the same kinetic energy minimization as in \cref{eq:benamou_brenier} but modifying the continuity equation to account for stochasticity,
\begin{align}
    \frac{1}{2}  W_\epsilon(p_0,p_1) = \inf_{v_t}\inf_{q_t} \int_0^1 \frac{1}{2} \mathbb{E}_{q_t(x)}\norm{v_t(x)}^2 dt , \;\;\text{ s.t.  } 
    \deriv{q_t(x)}{t} = 
    -\nabla \cdot \big( q_t(x) v_t(x) \big)
    + \frac{\sigma_t^2}{2}\Delta q_t(x)\, \text{  and } q_0 = p_0, \; q_1 = p_1\,. 
    \label{eq:entropic_ot}
\end{align}
Since we fix the density path $q_t$ in our Action Matching formulation problem, we again omit the optimization over $q_t$ and the marginal constraints,
\begin{align}
    \mathcal{K}_{\text{eAM}}
    \coloneqq  \inf_{v_t} \int_0^1 \frac{1}{2}  \mathbb{E}_{q_t(x)}\norm{v_t(x)}^2 dt , \;\;\text{ s.t.  } 
    \deriv{q_t(x)}{t} = 
    -\nabla \cdot \big( q_t(x) v_t(x) \big)
    + \frac{\sigma_t^2}{2}\Delta q_t(x) \label{eq:k_eam}
\end{align}
Introducing Lagrange multiplier $s_t$ to enforce the Fokker-Planck constraint leads to
\begin{align}
\mathcal{K}_{\text{eAM}}
&=\sup \limits_{s_t} \inf \limits_{v_t}  \int_0^1 \frac{1}{2}  \mathbb{E}_{q_t(x)}\norm{v_t(x)}^2 dt  + 
   \int \limits_0^1  \int   s_t(x_t) \Big( \frac{\partial q_t}{\partial t}(x_t) + 
   \nabla \cdot \big( q_t(x_t) v_t(x_t) \big)
   - \frac{\sigma_t^2}{2}\Delta q_t(x) \Big) dx_t dt . \nonumber
\end{align}
Compared to \cref{eq:ot_lagr} and \cref{eq:lagr_derivation_block}, note that the additional $- s_t(x_t) \frac{\sigma_t^2}{2}\Delta q_t(x)$ term does not depend on $v_t$.   Thus, integrating by parts and eliminating $v_t$ as above yields the condition that $v_t^*$ is a gradient field,
\begin{align}
v_t(x) = \nabla s_t^*(x) %
\end{align}
Substituting into \cref{eq:k_eam} and following derivations as in the proof of \cref{prop:eam_loss} in \cref{app:eam_proof} yields the entropic AM objective, 
\begin{align}
    \mathcal{L}_{\text{eAM}} \coloneqq \inf_{s_t(x)} \; \int dx\; s_0(x) p_0(x) - \int dx\; s_1(x) p_1(x) + \int \limits_0^1 \int \limits_{\mathcal{X}} dt dx \; q_t(x) \;  \bigg[\frac{1}{2} \| \nabla s_t \|^2 + \frac{\sigma_t^2}{2} \Delta s_t + \deriv{s_t}{t}\bigg].
\end{align}

Since $\nabla \tilde{s}_t^*$ is the unique gradient field which satisfies the Fokker-Planck equation for the distributional path of $q_t$ (\cref{prop:entropic_am}), and the solution $v_t^*$ minimizing the kinetic energy in \cref{eq:k_eam} is a gradient field satisfying the Fokker-Planck equation, we conclude that these vector fields are the same, $v_t^* = \nabla \tilde{s}_t^*$.

\subsection{Unbalanced Action Matching and Unbalanced Optimal Transport}\label{app:infinitesimal_uam}

To account for growth or destruction of probability mass across time or optimal transport between positive measures with unequal normalization constant, \citep{chizat2018interpolating, chizat2018unbalanced, kondratyev2016new, liero2016optimal,liero2018optimal} analyze optimal transport problems involving a growth rate $g_t(x)$.    
In particular, the `Wasserstein Fisher-Rao' distance \citep{chizat2018interpolating} is defined by adding a term involving the norm of the growth rate $g_t$ to the \citet{benamou2000computational} dynamical OT formulation in \cref{eq:benamou_brenier} and accounting for the growth term in the modified continuity equation (see \cref{eq:iw_ode}-\ref{eq:growth_continuity})
\begin{align}
   WFR_{\lambda}(p_0, p_1) \coloneqq \inf_{v_t}\inf_{g_t}\inf_{q_t}  \int_0^1 &\; \mathbb{E}_{q_t(x)} 
    \left[ \frac{1}{2} \norm{v_t(x)}^2 + \frac{\lambda}{2}  g_t(x)^2 \right] dt ,\;\;\text{ subj. to} \label{eq:unbalanced_vt_ot} \\
    \deriv{q_t(x)}{t} &= 
    -\nabla \cdot \big( q_t(x) v_t(x) \big)
    + \lambda  g_t(x)q_t(x)\,,  \quad \text{ and } q_0 = p_0, \; q_1 = p_1. \label{eq:modified_continuity}
\end{align}
where the growth term may also be scaled by a multiplier $\lambda$.
If we again fix the path $q_t$, as in Action Matching, we define 
\begin{align}
\mathcal{K}_{\text{uAM}}^* + \mathcal{G}_{\text{uAM}}^* \coloneqq \inf_{v_t}\inf_{g_t}  \int_0^1 &\; \mathbb{E}_{q_t(x)} 
    \left[ \frac{1}{2} \norm{v_t(x)}^2 + \frac{\lambda}{2}  g_t(x)^2 \right] dt \, \, \text{   s.t.  } \, \deriv{q_t(x)}{t} = 
    -\nabla \cdot \big( q_t(x) v_t(x) \big)
    + \lambda g_t(x)q_t(x) \label{eq:k_uam}
\end{align}
Note that \cref{eq:unbalanced_vt_ot} and \cref{eq:k_uam} are each convex optimizations after the change of variables $(q_t, v_t, g_t) \mapsto (q_t,~ q_tv_t, ~q_tg_t)$ \citep{chizat2018interpolating}.  We slightly abuse notation to write integrals as expectations with respect to $q_t(x)$, which may not be a normalized probability measure in general (see below, \citet{liero2016optimal}).

Introducing $s_t$ as a Lagrange multiplier enforcing the modified continuity equation, we obtain the following Lagrangian
\small
\begin{align}
   \mathcal{K}_{\text{uAM}}^* + \mathcal{G}_{\text{uAM}}^* &= \sup_{s_t}\inf_{v_t}\inf_{g_t} \; \int_0^1 \; \mathbb{E}_{q_t(x)} 
    \left[ \frac{1}{2} \norm{v_t(x)}^2 + \frac{\lambda}{2}  g_t(x)^2 \right] dt + \int \limits_0^1  \int   s_t(x) \Big( \frac{\partial q_t}{\partial t}(x) + 
    \nabla \cdot \big( q_t(x) v_t(x) \big)
    - q_t(x) g_t(x) \Big) dx dt \nonumber \\
    &\overset{(1)}{=} \sup_{s_t}\inf_{v_t}\inf_{g_t} \int_0^1 \int \; %
    \left[ \frac{1}{2} \norm{v_t(x)}^2 + \frac{\lambda}{2}  g_t(x)^2 \right] q_t(x) dx dt +  \int  s_t(x) q_t(x) dx\Big|^{t=1}_{t=0} - \int\limits_0^1 \int  \frac{\partial s_t(x)}{\partial t} q_t(x) dx dt  \label{eq:uot_lagr} \\
    &\phantom{\overset{(1)}{=} \sup_{s_t(x)}\inf_{v_t}\inf_{g_t}}   +  \int_0^1 \int \Big[ -\lambda 
 s_t(x) ~ g_t(x) - \langle v_t(x), \nabla s_t(x) \rangle  \Big]~ q_t(x) dx \nonumber
\end{align}
\normalsize
where integrate by parts wrt $x$ and $t$ in $(1)$.    
Finally, eliminating $v_t$ and $g_t$ yields the optimality conditions
\begin{align}
v_t^*(x) = \nabla s_t(x) \qquad g_t^*(x) =  s_t(x).
\end{align}
These optimality conditions show that the action function $s_t$ links the problems of transporting mass via the vector field $v_t$ and creating or destroying mass via the growth rate $g_t$.   
Substituting back into \cref{eq:uot_lagr} and simplifying, we obtain the unbalanced objective $\mathcal{L}_{\text{uAM}}(s_t)$, which is to be minimized as a function of the variational action 
\begin{align}
\mathcal{L}_{\text{{uAM}}_{\lambda}}(s_t)&\coloneqq \EE_{q_{{0}}\left(x\right)}\left[s_{{0}}\left(x\right)\right]-\EE_{q_{{1}}\left(x\right)}\left[s_{{1}}\left(x\right)\right] +\int\EE_{q_{t}\left(x\right)}\left[\fr 12\norm{\g s_{t}\left(x\right)}^{2}+\fr{\pa s_{t}}{\pa t}\left(x\right)+\fr{\lambda}{2} s_t^2 \right]dt\,
\end{align}
As in \cref{prop:uam_loss}, we can define an appropriate action gap involving the growth term, with the optimal $\hat{s}_t^*$ evaluating the kinetic energy $\mathcal{K}_{\text{uAM}}^*(\g \hat{s}_t^*)\coloneqq\fr 12\int_0^1 \EE_{q_{t}\left(x\right)}\norm{\g \hat{s}_{t}^{*}(x)}^{2}dt$ and growth energy $\mathcal{G}_{\text{uAM}}^*(\hat{s}_t^*)\coloneqq\frac{\lambda}{2}\int_0^1\EE_{q_{t}\left(x\right)}[ \hat{s}_{t}^{*}(x)^{2}]dt$ of a given curve $q_t$.

\paragraph{Metric Derivative} 
Finally, the uAM objective and \cref{eq:k_uam} are related to the \textit{metric derivative} in the space of finite positive Borel measures
$\mathcal{M}(\mathcal{X})$ equipped with the Wasserstein Fisher-Rao metric distance in \cref{eq:unbalanced_vt_ot}, where we consider measures with arbitrary mass due to the effect of the growth term.   \citet{liero2016optimal} Thm. 8.16 and 8.17 are analogous to \citet{ambrosio2008gradient} Thm. 8.3.1 for this case.   In particular, for 
$q_t, q_{t+dt}$ along an absolutely continuous (AC) curve of positive measures with $\partial_t q_t= -\nabla \cdot (q_t v_t^*) + q_t g_t^*$, we have \citep{liero2016optimal}
\small
\begin{align}
   \hspace*{-.2cm} |q_t^\prime|^2 = \left( \lim \limits_{dt \rightarrow 0} \frac{WFR_{1}(q_{t}, q_{t+dt})}{dt}\right)^2 = \int \big[ \|v_t^*(x)\|^2 +   \| g_t^*(x)\|^2 \big] q_t(x) dx = \int \big[ \|\nabla \hat{s}_t^*(x)\|^2 +  \| \hat{s}_t^*(x)\|^2 \big] q_t(x) dx 
\end{align}
\normalsize
The optimal uAM objective (\cref{eq:k_uam}) at each $t$ thus reflects the infinitesimal squared WFR distance along an AC curve.

\subsection{Action Matching and Wasserstein Gradient Flows}\label{app:grad_flows}
We can also view dynamics learned by action matching as parameterizing the gradient flow of a time-dependent linear functional on the Wasserstein-2 manifold.  This is in contrast \text{JKOnet} \citep{bunne2022proximal}, which is limited to learning gradient flows for class of time-homogenous linear functionals.
We give a brief review of concepts related to gradient flows, but refer to e.g. \citet{figalli2021invitation} Ch. 3-4 for more details. 
Throughout this section, we let $\mathcal{X} \subseteq \mathbb{R}^d$, consider the space of probability measures $\mathcal{P}_2(\mathcal{X})$ with finite second moment, and identify measures $\mu \in \mathcal{P}_2(\mathcal{X})$ with their densities $d\mu = p dx$.

In order to define a gradient flow, we first need to define an inner product on the tangent space.   
The seminal work of \citet{otto2001geometry} defines the desired Riemannian metric on $T\mathcal{P}_2(\mathcal{X})$. 
Consider two curves $p_t^{(i)}: t \rightarrow \mathcal{P}_2(\mathcal{X})$ passing through a point $q \in \mathcal{P}_2(\mathcal{X})$, with $p_0^{(i)} = q$ and tangent vectors $\dot{p_0}^{(1)}, \dot{p_0}^{(2)} \in T_q \mathcal{P}_2(\mathcal{X})$ for $\dot{p}^{(i)} \coloneqq \frac{\partial p^{(i)}_t}{\partial t}$. 
Using \cref{th:ac_cont_short}, each curve satisfies a continuity equation for a gradient field $\nabla \psi_t^{(i)}(x)$,   
e.g. $\dot{p_t}^{(i)}|_{t=0} = - \divp{p_0^{(i)} \nabla \psi_0^{(i)}}$ at $t=0$.  The inner product on $T_q \mathcal{P}_2(\mathcal{X})$ is defined as 
\begin{align}
\langle \dot{p_t}^{(1)}, \dot{p_t}^{(2)} \rangle_{q} = \int \langle \nabla \psi_0^{(1)}(x), \nabla \psi_0^{(2)}(x) \rangle~ q(x) dx \quad \qquad \big( \text{where} \quad \dot{p_t}^{(i)}|_{t=0} = 
-\divp{p_0^{(i)} ~ \nabla \psi_0^{(i)}} \big)
\end{align}
The gradient of a functional $\mathcal{F}[p_t]$ with respect to the Wasserstein-2 metric is defined as the tangent vector for which the inner product yields the directional derivative along a curve $p_t : t \rightarrow \mathcal{P}_2(\mathcal{X})$ at $t=0$
\begin{align}
\hspace*{-.2cm} \langle &\text{grad}_{W_2} \mathcal{F}[p_t], \dot{p_t} \rangle_{p_0} = \frac{d}{d t}\Big|_{t=0} \mathcal{F}[p_t] \\[1.5ex] 
\hspace*{-.2cm} \text{or, more explictly, as: } \qquad &\text{grad}_{W_2}\mathcal{F}[p_t] = 
- \nabla \cdot \Big(  p_t ~ \nabla \frac{\delta \mathcal{F}[p_t]}{\delta p_t} \Big) ,
\end{align}
where $\frac{\delta \mathcal{F}[p]}{\delta p}$ is the first variation.   For example, the Wasserstein gradient of a \textit{time-dependent} linear functional is given by
\begin{align}
\mathcal{F}_t[p_t] = \int p_t(x) s_t(x) dx \qquad \qquad \text{grad}_{W_2}\mathcal{F}_t[p_t] = 
- \nabla \cdot \big( p_t ~ \nabla s_t \big). 
\end{align}
A negative gradient flow on the Wasserstein manifold is then given, in either continuous or discrete time, by
\begin{align}
\frac{\partial p_t}{\partial t} = \text{grad}_{W_2}\mathcal{F}_t[p_t] = -
 \nabla \cdot \big( p_t ~ \nabla s_t \big).\label{eq:grad_flow} 
\end{align}
Thus, the negative gradient flow of a time-dependent linear functional on $\mathcal{P}_2(\mathcal{X})$ can be modeled using the continuity equation for a vector field $\nabla s_t$.   
Action Matching can now be viewed as learning the functional $\mathcal{F}_t[p_t] = \int p_t(x) s_t(x) dx$ for which a given density path $p_t : t \rightarrow \mathcal{P}_2(\mathcal{X})$ (or $q_t$, in the main text) traces gradient flow on the Wasserstein manifold $\mathcal{P}_2(\mathcal{X})$.

\paragraph{Comparison with \textsc{JKOnet}}
In discrete time, the gradient flow in \cref{eq:grad_flow} can be written as
\begin{align}
p_{t+1} = \argmin \limits_{p}  \mathcal{F}[p] + \frac{1}{2\tau} W_2^2(p, p_{t}) .
\end{align}
 where we write the Wasserstein distance between densities instead of measures.  
The \textsc{JKOnet} method of \citet{bunne2022proximal} uses the above discrete-time approach to learn a potential function $\mathcal{F}[p]$ which drives the observed sample dynamics.    However, they restrict attention to learning the parameters $\theta$ of a \textit{time-homogeneous} linear functional $\mathcal{F}[p] = \int p(x) s(x;\theta) dx$ and their optimization methodology is notably more complex than action matching.

\section{Generative Modeling in Practice}
\label{app:generative_modeling}

\begin{algorithm}
  \caption{Generative Modeling using Action Matching \textcolor{azure(colorwheel)}{(In Practice)}}
  \begin{algorithmic}
    \REQUIRE dataset $\{x^j_t\}_{j=1}^{N_t}, \; x^j_t \sim q_t(x)$, batch-size $n$
    \REQUIRE parameteric model $s_t(x,\theta)$, \textcolor{azure(colorwheel)}{weight schedule $\omega(t)$}
    \FOR{learning iterations}
        \STATE get batch of samples from boundaries: $\{x_0^\ii\}_{i=1}^n \sim q_0(x), \; \{x_1^\ii\}_{i=1}^n \sim q_1(x)$ \\

        \STATE sample times $\{t^i\}_{i=1}^n \sim \textcolor{azure(colorwheel)}{p(t)}$
        \STATE get batch of intermediate samples $\{x_{\sampledt}^\ii\}_{i=1}^n \sim q_{t}(x)$
        \STATE $\text{L}_i =  \bigg[s_{0}(x_0^i)\textcolor{azure(colorwheel)}{\omega(0)} - s_{1}(x_1^i)\textcolor{azure(colorwheel)}{\omega(1)} + \frac{1}{2}\norm{\nabla s_t(x_{t^i}^i)}^2\textcolor{azure(colorwheel)}{\omega(t^i)} + \deriv{s_t(x_{t^i}^i)}{t}\textcolor{azure(colorwheel)}{\omega(t^i)} + s_t(x_{t^i}^i)\textcolor{azure(colorwheel)}{\deriv{\omega(t^i)}{t^i}}\bigg]$
        \STATE $\text{L} = \sum_{i=1}^n\textcolor{azure(colorwheel)}{\frac{1}{p(t^i)}}\text{L}_i$
        \STATE update the model $\theta \gets \text{Optimizer}(\theta, \nabla_\theta\text{L}_\theta)$
    \ENDFOR
    \OUTPUT trained model $s_t(x,\theta^*)$
  \end{algorithmic}
  \label{alg:amgl_practice}
\end{algorithm}

In practice, we found that the naive application of Action Matching (Algorithm \ref{alg:amgl}) for complicated dynamics such as image generation might exhibit poor convergence due to the large variance of objective estimate.
Moreover, the optimization problem
\begin{align}
    \min_{s_t} \fr 12\int \mean_{q_{t}\left(x\right)}\norm{\g s_{t}\left(x\right)-\g s_{t}^{*}\left(x\right)}^{2}dt
\end{align}
might be ill posed due to the singularity of the ground truth vector field $\g s_t^*$.
This happens when the data distribution $q_0$ is concentrated close to a low dimensional manifold, and the final distribution $q_1$ has a much higher intrinsic dimensionality (e.g., Gaussian distributions). In this case, the deterministic velocity vector field must be very large (infinite in the limit), so that it can pull apart the low dimensional manifold to transform it to higher dimensions.

We now discuss an example of this behavior, when the data distribution is a mixture of delta functions.
Consider the sampling process
\begin{align}
    x_t = f_t(x_0) + \sigma_t \eps, \;\; x_0 \sim \pi(x), \;\; \eps \sim \Normal(x\cond 0,1)\,,
\end{align}
where the target distribution is a mixture of delta-functions
\begin{align}
    \pi(x) = \frac{1}{N}\sum_i^N \delta(x-x^i).
\end{align}
Denoting the distribution of $x_t$ as $q_t(x)$, we can solve the continuity equation
\begin{align}
    \deriv{q_t}{t} = - \diver{q_t v_t}
\end{align}
analytically (see Appendix \ref{app:sparse_data}) by finding one of the many possible solutions
\begin{align}
    v_t =~&\frac{1}{\sum_i q_t^i(x)}\sum_i q_t^i(x)\bigg[(x-f_t(x^i))\deriv{}{t}\log \sigma_t +\deriv{f_t(x^i)}{t} \bigg], \;\; q^i_t(x) = \Normal(x\cond f_t(x^i), \sigma_t^2).
\end{align}
Note that $v_t$ is not gradient field in general, and thus is not the solution of action matching. However, it can be written as
\begin{align}
    v_t(x) =~&\sum_i \frac{q_t^i(x)}{\sum_i q_t^i(x)} \g s_t^i(x), \quad \text{where}\quad s^i_t(x) = \frac{1}{2}(x-f_t(x^i))^2\deriv{}{t}\log \sigma_t +\bigg\langle\deriv{f_t(x^i)}{t},x\bigg\rangle.\nonumber
\end{align}
Given that the density of Gaussian distributions drop exponentially fast, we can conclude that for small values of $t$ around each $x^i$, $\frac{q_t^j(x)}{\sum_j q_t^j(x)}$ is close to $1$ if $i=j$, and close to $0$ if $i\ne j$. Thus, $v_t(x)$ around each $x^i$ can be locally approximated with the gradient vector field $\g s_t^i(x)$.
Now suppose $\g s^*_t(x)$ is the solution of action matching, i.e., the unique gradient vector field that solves the continuity equation in every region, including regions around each $x^i$. Given the uniqueness of gradient vector fields that solve continuity equation, we can conclude that $\g s_t^i(x)$ locally matches $\g s^*_t(x)$ around each $x^i$.

For generative modeling, it's essential that $q_0 = \pi(x)$; hence, $\lim_{t\to 0}\sigma_t = 0$ and $\lim_{t\to 0}f_t(x) = x$.
Assuming that $\sigma_t^2$ is continuous and differentiable at $0$, in the limit, around each $x^i$, we have
\begin{align}
    \text{ for } \; t\to 0\,,\;\; \norm{\g s_{t}^*\left(x\right)}^{2} \propto \frac{1}{\sigma_t^2}, \;\; \text{ and}\;\; \fr 12\mean_{q_{t}\left(x\right)}\norm{\g s_{t}^*\left(x\right)}^{2} \propto \frac{1}{\sigma_t^2}.
\end{align}
Thus, the loss can be properly defined only on the interval $t\in (\delta,1]$, where $\delta > 0$.
In practice, we want to set $\delta$ as small as possible, i.e., we ideally want to learn $s_t$ on the whole interval $t\in[0,1]$.
We can prevent learning the singularity functions just by re-weighting the objective in time as follows
\begin{align}\label{appeq:weighted_objective}
    \fr 12\int \mean_{q_{t}\left(x\right)}\norm{\g s_{t}\left(x\right)-\g s_{t}^{*}\left(x\right)}^{2}dt \longrightarrow \fr 12\int \textcolor{azure(colorwheel)}{\omega(t)} \mean_{q_{t}\left(x\right)}\norm{\g s_{t}\left(x\right)-\g s_{t}^{*}\left(x\right)}^{2}dt.
\end{align}
To give an example, we can take $\sigma_t = \sqrt{t}$ and $f_t(x) = x\sqrt{1-t}$, then $\omega(t) = (1-t)t^{3/2}$ cancels out the singularities at $t=0$ and $t=1$.

The second modification of the original Algorithm \ref{alg:amgl} is the sampling of time-steps for the estimation of the time integral.
Namely, the optimization of \Cref{appeq:weighted_objective} is equivalent to the minimization of the following objective
\begin{align}
    \L_{\text{AM}} \left(s\right)=~&\underbrace{ \omega(1)\mean_{q_{1}\left(x\right)}\left[s_{1}\left(x\right)\right]-\omega(0)\mean_{q_{0}\left(x\right)}\left[s_{0}\left(x\right)\right]\vphantom{\left[\fr{\pa}{\pa}\right]}}_{\text{weighted action-increment}}  \\
    &+\underbrace{\int_{0}^{1}\mean_{q_{t}\left(x\right)}\left[\fr 12\omega(t)\norm{\g s_{t}\left(x\right)}^{2}+\omega(t)\fr{\pa s_{t}\left(x\right)}{\pa t} + s_t(x)\fr{\pa \omega(t)}{\pa t}\right]dt}_{\text{weighted smoothness}}\,,
\end{align}
which consists of two terms.
Estimation of the weighted action-increment involves only sampling from $q_{0}$ and $q_{1}$, while the weighted smoothness term estimate depends on the distribution of time samples $p(t)$, i.e.,
\begin{align}
    \int_{0}^{1}\underbrace{\textcolor{azure(colorwheel)}{\frac{p(t)}{p(t)}}}_{\text{=1}}\mean_{q_{t}\left(x\right)}\left[\fr 12\omega(t)\norm{\g s_{t}\left(x\right)}^{2}+\omega(t)\fr{\pa s_{t}\left(x\right)}{\pa t} + s_t(x)\fr{\pa \omega(t)}{\pa t}\right]dt  \\
    = \mean_{t\sim \textcolor{azure(colorwheel)}{p(t)}} \mean_{x\sim q_t(x)} \textcolor{azure(colorwheel)}{\frac{1}{p(t)}}\left[\fr 12\omega(t)\norm{\g s_{t}\left(x\right)}^{2}+\omega(t)\fr{\pa s_{t}\left(x\right)}{\pa t} + s_t(x)\fr{\pa \omega(t)}{\pa t}\right].
\end{align}
Note that $p(t)$ can be viewed as a proposal importance sampling distribution, and thus every choice of it results in an unbiased estimate of the original objective function.
Thus, we can design $p(t)$ to reduce the variance of the weighted smoothness term of the objective.
In our experiments, we observed that simply taking $p(t)$ proportionally to the standard deviation of the corresponding integrand significantly reduces the variance, i.e.,
\begin{align}
    p(t) \propto \sqrt{\mean_{x\sim q_t}(\zeta_t - \mean_{x\sim q_t}\zeta_t)^2},\;\;  \zeta_t = \fr 12\omega(t)\norm{\g s_{t}\left(x\right)}^{2}+\omega(t)\fr{\pa s_{t}\left(x\right)}{\pa t} + s_t(x)\fr{\pa \omega(t)}{\pa t}.
\end{align}
We implement sampling from this distribution by aggregating the estimated variances throughout the training with exponential moving average, and then followed by linear interpolation between the estimates.

\section{Sparse Data Regime}
In this section, we find velocity vector fields that satisfy the continuity equation in the case where the data distribution $q_0$ is a delta function or a mixture of delta functions; and the conditional $k_t(x_t\cond x)$ is a Gaussian distribution.
\label{app:sparse_data}
\subsection{Delta Function Data Distribution}
We start with the case where the dataset consists only of a single point $x_0 \in \mathbb{R}^d$
\begin{align}
    q_0(x) = \delta(x-x_0), \;\;\; k_t(x_t\cond x) = \Normal(x_t\cond f_t(x),\sigma_t^2).
\end{align}
Then the distribution at time $t$ is
\begin{align}
    q_t(x) = \int dx'\; q_0(x')k_t(x\cond x') = \Normal(x\cond f_t(x_0),\sigma_t^2).
\end{align}
The ground truth vector field $v$ comes from the continuity equation
\begin{align}
    \deriv{q_t}{t} = - \diver{q_t v} \;\implies\; \deriv{}{t}\log q_t = - \langle\nabla\log q_t, v\rangle -  \diver{v}.
\end{align}
For our dynamics, we have
\begin{align}
    \deriv{}{t}\log q_t =~& \deriv{}{t}\bigg[-\frac{d}{2}\log(2\pi\sigma_t^2) - \frac{1}{2\sigma_t^2}\norm{x-f_t(x_0)}^2\bigg] \\
    =~& - d\deriv{}{t}\log\sigma_t + \frac{1}{\sigma_t^2}\norm{x-f_t(x_0)}^2\deriv{}{t}\log\sigma_t + \frac{1}{\sigma_t^2}\bigg\langle x-f_t(x_0), \deriv{f_t(x_0)}{t}\bigg\rangle\\
     =~& - d\deriv{}{t}\log\sigma_t + \frac{1}{\sigma_t^2}\bigg\langle x-f_t(x_0), (x-f_t(x_0))\deriv{}{t}\log\sigma_t + \deriv{f_t(x_0)}{t}\bigg\rangle;\\
     \nabla\log q_t =~& -\frac{1}{\sigma_t^2}(x-f_t(x_0));\\
     \deriv{}{t}\log q_t =~&- d\deriv{}{t}\log\sigma_t -\bigg\langle \nabla\log q_t, (x-f_t(x_0))\deriv{}{t}\log\sigma_t + \deriv{f_t(x_0)}{t}\bigg\rangle.
\end{align}
Matching the corresponding terms in the continuity equation, we get
\begin{align}
    v = (x-f_t(x_0))\deriv{}{t}\log\sigma_t + \deriv{f_t(x_0)}{t}.
\end{align}
We note that since the above vector field is gradient field, it is the unique vector field that the action matching would recover.
\subsection{Mixture of Delta Functions Data Distribution}
For the mixture of delta-functions, we denote
\begin{align}
    q_0(x) = \frac{1}{N}\sum_i^N \delta(x-x^i), \;\; q_t(x) = \frac{1}{N}\sum_i^N q^i_t(x), \;\; q^i_t(x) = \Normal(x\cond f_t(x^i), \sigma_t^2).
\end{align}
Due to the linearity of the continuity equation w.r.t. $q$, we have
\begin{align}
    \sum_i \deriv{q^i_t}{t} = \sum_i  \diver{q^i_t v} \;\implies\; \sum_i q^i_t\bigg(\deriv{}{t} \log q^i_t + \langle\nabla\log q^i_t, v\rangle +  \diver{v}\bigg) = 0.
    \label{appeq:mix_continuity}
\end{align}
We first solve the equation for $\deriv{f_t}{t} = 0$, then for $\deriv{}{t}\log \sigma_t = 0$ and join the solutions.

For $\deriv{f_t}{t} = 0$, we look for the solution in the following form
\begin{align}
    v_\sigma = \frac{A}{\sum_i q^i_t} \sum_{i}\nabla q^i_t, \;\; q^i_t(x) = \Normal(x\cond f^i_t(x^i), \sigma_t^2).
\end{align}
Then we have
\begin{align}
    \diver{v_\sigma} =~& \bigg\langle \nabla \frac{A}{\sum_i q^i_t}, \sum_i \nabla q^i_t\bigg\rangle + \frac{A}{\sum_i q^i_t} \sum_i \nabla^2 q^i_t \\
    =~& - \frac{A}{(\sum_i q^i_t)^2} \norm{\sum_i \nabla q^i_t}^2 + \frac{A}{\sum_i q^i_t}\sum_i q^i_t\bigg[\norm{\nabla\log q^i_t}^2 - \frac{d}{\sigma_t^2}\bigg],\\
    \big(\sum_i q^i_t\big) \diver{v_\sigma} =~& - \frac{A}{\sum_i q^i_t} \norm{\sum_i \nabla q^i_t}^2 + A\sum_i q^i_t\bigg[\norm{\nabla\log q^i_t}^2 - \frac{d}{\sigma_t^2}\bigg],
\end{align}
and from \eqref{appeq:mix_continuity} we have
\begin{align}
    \sum_i q^i_t\bigg(-d \deriv{}{t} \log \sigma_t + \bigg\langle \nabla\log q^i_t, v_\sigma + \sigma_t^2\deriv{}{t}\log \sigma_t \nabla\log q^i_t\bigg\rangle +  \diver{v_\sigma}\bigg) = 0.
\end{align}
From these two equations we have
\begin{align}
    \sum_i q^i_t \diver{v_\sigma} &~=-\frac{A}{\sum_i q^i_t} \norm{\sum_i \nabla q^i_t}^2 + A\sum_i q^i_t\bigg[\norm{\nabla\log q^i_t}^2 - \frac{d}{\sigma_t^2}\bigg] = \\
    &~= \sum_i q^i_t \bigg(d \deriv{}{t} \log \sigma_t\bigg) - \frac{A}{\sum_i q^i_t}\norm{\sum_i \nabla q^i_t}^2 - \sigma_t^2\deriv{}{t}\log \sigma_t\sum_i q^i_t\norm{\nabla\log q^i_t}^2.
\end{align}
Thus, we have
\begin{align}
    A = -\sigma_t^2\deriv{}{t}\log \sigma_t.
\end{align}
For $\deriv{}{t}\log \sigma_t = 0$, we simply check that the solution is
\begin{align}
    v_f = \frac{1}{\sum_i q^i_t}\sum_i q^i_t \deriv{f_t(x^i)}{t}.
    \label{appeq:v_f_solution}
\end{align}
Indeed, the continuity equation turns into
\begin{align}
    \sum_i q^i_t \bigg(\bigg\langle\nabla\log q^i_t, v_f-  \deriv{f_t(x^i)}{t}\bigg\rangle +  \diver{v_f}\bigg) = 0.
\end{align}
From the solution and the continuity equation we write $\sum_i q^i_t \diver{v_f}$ in two different ways.
\begin{align}
    \sum_i q^i_t \diver{v_f} =~& -\frac{1}{\sum_i q^i_t}\bigg\langle\sum_i \nabla q^i_t, \sum_i q^i_t \deriv{f_t(x^i)}{t}\bigg\rangle + \sum_i\bigg\langle\nabla q^i_t, \deriv{f_t(x^i)}{t}\bigg\rangle\\
    =~& - \bigg\langle\sum_i\nabla q^i_t, v_f\bigg\rangle + \sum_i \bigg\langle\nabla q^i_t, \deriv{f_t(x^i)}{t}\bigg\rangle
\end{align}
Thus, we see that \eqref{appeq:v_f_solution} is indeed a solution.

Finally, unifying $v_\sigma$ and $v_f$, we have the full solution
\begin{align}
    v =~& -\bigg(\deriv{}{t}\log \sigma_t\bigg)\frac{\sigma_t^2}{\sum_i q^i_t} \sum_{i}\nabla q^i_t + \frac{1}{\sum_i q^i_t}\sum_i q^i_t \deriv{f_t(x^i)}{t}, \;\; q^i_t(x) = \Normal(x\cond f_t(x^i), \sigma_t^2),\\
    v =~&\frac{1}{\sum_i q^i_t}\sum_i q^i_t\bigg[(x-f_t(x^i))\deriv{}{t}\log \sigma_t +\deriv{f_t(x^i)}{t} \bigg].
\end{align}

\section{Experiments Details}

\subsection{Schrödinger Equation Simulation}
\label{app:schrodinger}

\begin{figure*}[t]
    \centering
    \includegraphics[width=0.8\textwidth]{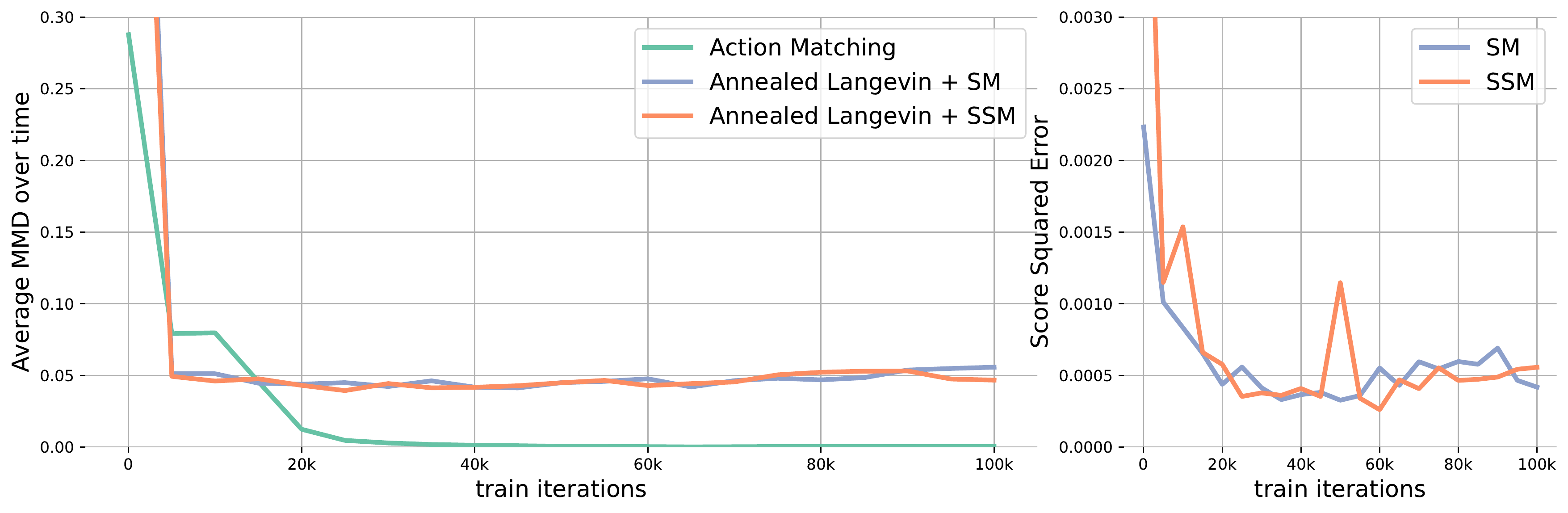}
    \caption{On the left, we demonstrate performance of various algorithms in terms of average MMD over the time of dynamics. The MMD is measured between generated samples and the training data. On the right, we report squared error of the score estimation for the score-based methods.}
    \label{fig:mmd_se}
\end{figure*}

For the initial state of the dynamics 
\begin{align}
    i\deriv{}{t}\psi(x,t) = -\frac{1}{\norm{x}}\psi(x,t) -\frac{1}{2}\nabla^2\psi(x,t)\,,
\end{align}
we take the following wavefunction
\begin{align}
     \psi(x,t=0) \propto \psi_{32-1}(x) + \psi_{210}(x), \;\;\text{ and } \;\; q_{t=0}^*(x) = |\psi(x,t=0)|^2,
\end{align}
where $n,l,m$ are quantum numbers and $\psi_{nlm}$ is the eigenstate of the corresponding Hamiltonian (see \cite{griffiths2018introduction}).
For all the details on sampling and the exact formulas for the initial state, we refer the reader to the code \href{https://github.com/necludov/action-matching}{github.com/necludov/action-matching}.
We evolve the initial state for $T=14\cdot 10^3$ time units in the system $\hbar=1, m_e=1, e=1, \eps_0 = 1$ collecting the dataset of samples from $q_t^*$. For the time discretization, we take $10^3$ steps; hence, we sample every $14$ time units.

To evaluate each method, we collect all the generated samples from the distributions $q_t,\; t\in [0,T]$ comparing them with the samples from the training data.
For the metric, we measure the Maximum Mean Discrepancy \citep{gretton2012kernel} between the generated samples and the training data at $10$ different timesteps $t=\frac{k}{10}T, \; k=1,\ldots,10$ and average the distance over the timesteps.
For the Annealed Langevin Dynamics, we set the number of intermediate steps for $M=5$, and select the step size $dt$ by minimizing MMD using the exact scores $\nabla \log q_t(x)$.

For all methods, we use the same architecture, which is a multilayer perceptron with $5$ layers $256$ hidden units each.
The architecture $h(t,x)$ takes $x \in \mathbb{R}^3$ and $t \in \mathbb{R}$ and outputs 3-d vector, i.e. $h(t,x): \mathbb{R}\times\mathbb{R}^3 \to \mathbb{R}^3$.
For the score-based models it already defines the score, while for action matching we use $s_t(x) = \norm{h(t,x)-x}^2$ as the model and the vector field is defined as $\nabla s_t(x)$.

In \cref{fig:mmd_se} we plot the convergence of Average MMD over time for Action Matching and the baselines. Despite that both SM and SSM accurately recover the ground truth scores for the marginal distributions (see the right plot in \cref{fig:mmd_se}), one cannot efficiently use them for the sampling from the ground truth dynamics.

\begin{algorithm}[h]
  \caption{Annealed Langevin Dynamics for the Schrödinger Equation}
  \begin{algorithmic}
    \REQUIRE score model $s_t(x) = \nabla\log q_t(x)$, step size $dt$, number of intermediate steps $M$
    \REQUIRE initial samples $x_0^i \in \mathbb{R}^d$
    \FOR{time steps $t \in (0,T]$}
        \FOR{intermediate steps $j \in 1,\ldots,M$}
            \STATE $\eps^i \sim \Normal(0,\mathbf{1})$
            \STATE $x_t^i = x_t^i + \frac{dt}{2}\nabla\log q_t(x_t^i) + \sqrt{dt}\cdot\eps^i$
        \ENDFOR
        \STATE save samples $x_t^i \sim q_t(x)$
    \ENDFOR
    \OUTPUT samples $\{x_t^i\}_{t=0}^T$
  \end{algorithmic}
  \label{alg:ALD}
\end{algorithm}

\newpage
\subsection{Generative Modeling}
\label{app:implementation}

For the architecture of the neural network parameterizing $s_t$, we follow \citep{salimans2021should} with a small modification. 
Namely, we parameterize $s_t(x)$ as $\inner{\text{unet}(t,x)}{x}$, where $\text{unet}(t,x)$ is the output of the U-net architecture \citep{ronneberger2015u}.
For the U-net architecture, we follow \citep{song2020score}.
We consider the same U-net architecture for the baseline to parameterize $\nabla\log q_t$.

For diffusion, we take VP-SDE from \citep{song2020score}, which corresponds to $\alpha_t = \exp(-\frac{1}{2}\int \beta(s)ds)$ and $\sigma_t = \sqrt{1-\exp(-\int \beta(s)ds)}$, where $\beta(s) = 0.1 + 19.9t$. 
All images are normalized to the interval $[-1, 1]$.

For the baseline, we managed to generate the images only taking into account the noise variance of the current distribution $q_t$ as proposed in \citep{song2019generative}. 
For propagating samples in time we select the time step $dt = 10^{-2}$ and perform $10$ sampling steps for every $q_t$.
We additionally run $100$ sampling steps for the final distribution. In total we run $1000$ steps to generate images.
See \cref{alg:ALD_baseline} for the pseudocode.

\begin{algorithm}[t]
  \caption{Annealed Langevin Dynamics for the Image Generation}
  \begin{algorithmic}
    \REQUIRE score model $s_t(x) = \nabla\log q_t(x)$, step size $dt$, number of intermediate steps $M$
    \REQUIRE initial samples $x_0^i \in \mathbb{R}^d$
    \FOR{time steps $t \in (0,1)$}
        \STATE $\alpha = \alpha_1(1-t)^2$
        \FOR{intermediate steps $j \in 1,\ldots,M$}
            \STATE $\eps^i \sim \Normal(0,\mathbf{1})$
            \STATE $x_t^i = x_t^i + \frac{\alpha}{2}\nabla\log q_t(x_t^i) + \sqrt{\alpha}\cdot\eps^i$
        \ENDFOR
    \ENDFOR
    \STATE $\alpha = \alpha_1$
    \FOR{intermediate steps $j \in 1,\ldots,M$}
        \STATE $\eps^i \sim \Normal(0,\mathbf{1})$
        \STATE $x_1^i = x_1^i + \frac{\alpha}{2}\nabla\log q_1(x_1^i) + \sqrt{\alpha}\cdot\eps^i$
    \ENDFOR
    \OUTPUT samples $x_1^i$
  \end{algorithmic}
  \label{alg:ALD_baseline}
\end{algorithm}

\begin{figure}
    \centering
    \includegraphics[width=\textwidth]{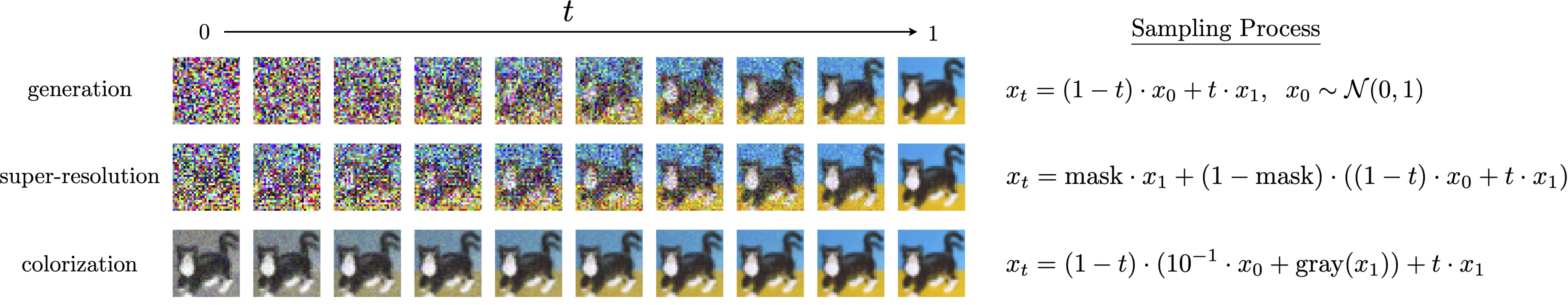}
    \caption{Examples of different noising processes used for different vision tasks. The processes interpolate between the prior distribution at $t=0$ and the target distribution $t=1$. For all the processes $x_0 \sim \Normal(0,1)$.}
    \label{fig:evolutions}
\end{figure}

\newpage
\subsection{Unbalanced Action Matching}
\label{app:uam}

We showcase Unbalanced Action Matching on a toy data for which we consider a mixture of gaussians, i.e., 
\begin{align}
    q_t(x) = \alpha_t \Normal(-5, 1) + (1-\alpha_t) \Normal(5, 1)\,, 
\end{align}
and change $\alpha_t$ linearly from $0.2$ to $0.8$.
In \cref{fig:uam}, we demonstrate the data samples and the samples generated by Unbalanced Action Matching starting from the ground truth samples at time $t=0$ and reweighting particles according to \cref{eq:uam_ode_solution_weights}.

Instead of attempting to transport particles from one mode to another, Unbalanced Action Matching is able to model this change of probability mass using the growth term $g_t(x) = s_t(x)$.

\begin{figure}[t]
    \centering
    \includegraphics[width=\textwidth]{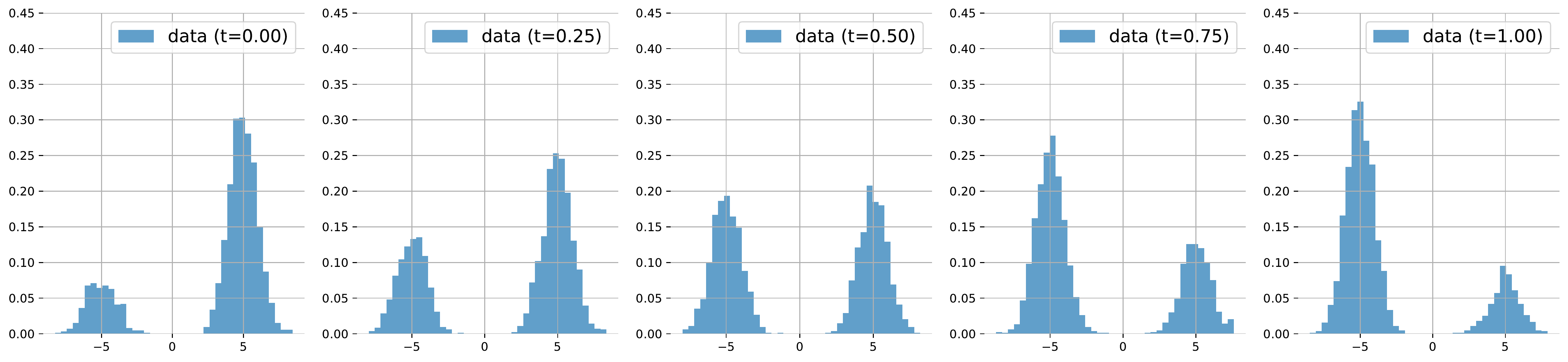}\\
    \includegraphics[width=\textwidth]{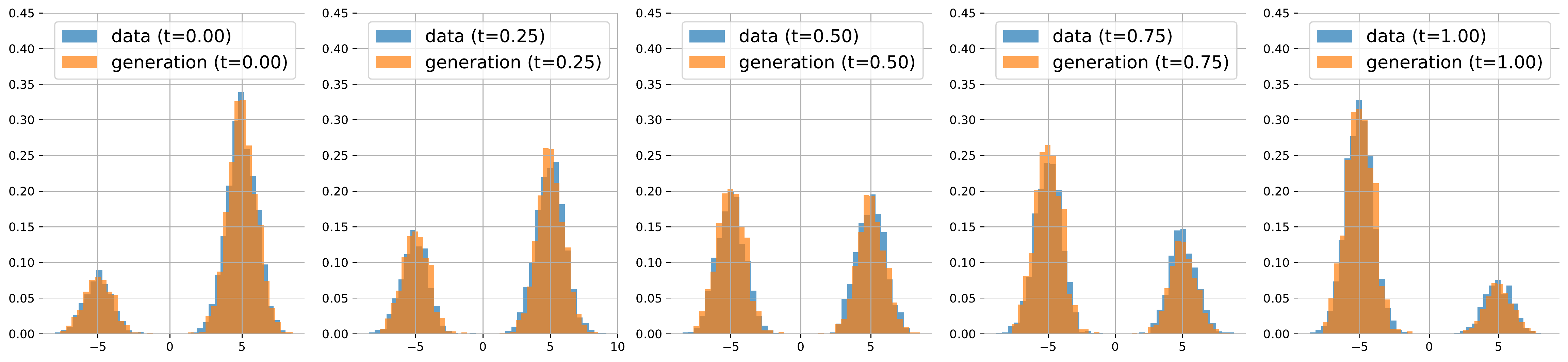}
    \caption{The histograms of the training data (top row) changing in time and the histograms of the generated samples by Unbalanced Action Matching (bottom row).}
    \label{fig:uam}
\end{figure}

\newpage
\section{Generated Images}
\label{app:images}

\begin{figure}[h]
    \centering
    \includegraphics[width=0.9\textwidth]{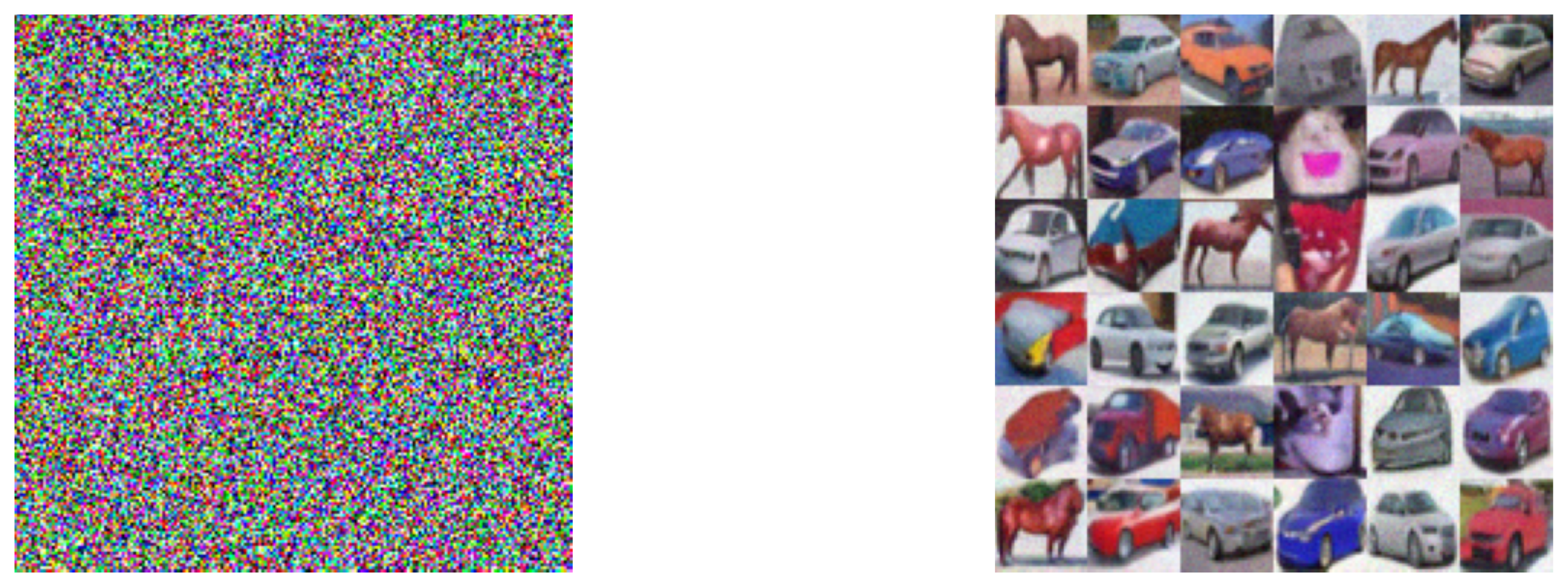}
    \caption{Images generated (right) by the baseline (ALD + SSM) from the noise (left).}
    \label{fig:baseline_pictures}
\end{figure}
\begin{figure}[h]
    \centering
    \includegraphics[width=0.9\textwidth]{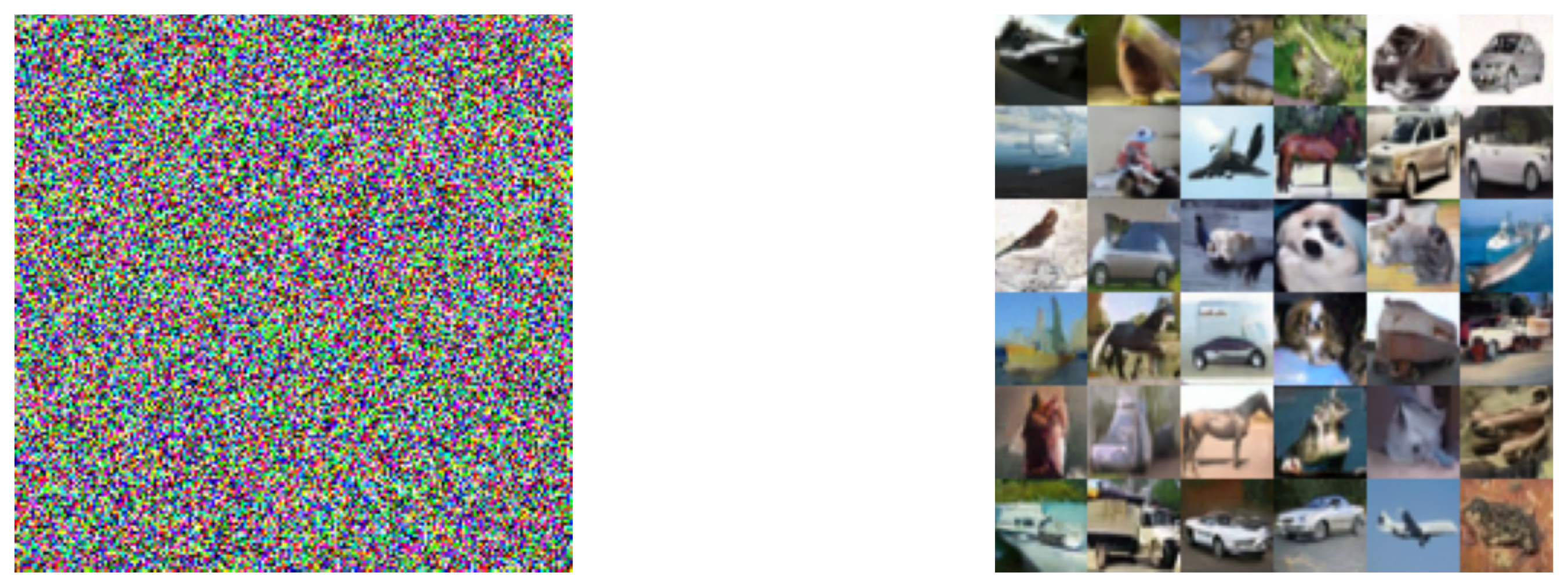}
    \caption{Images generated (right) by Action Matching from the noise (left).}
    \label{fig:am_generation_pictures}
\end{figure}
\begin{figure}[h]
    \centering
    \includegraphics[width=0.9\textwidth]{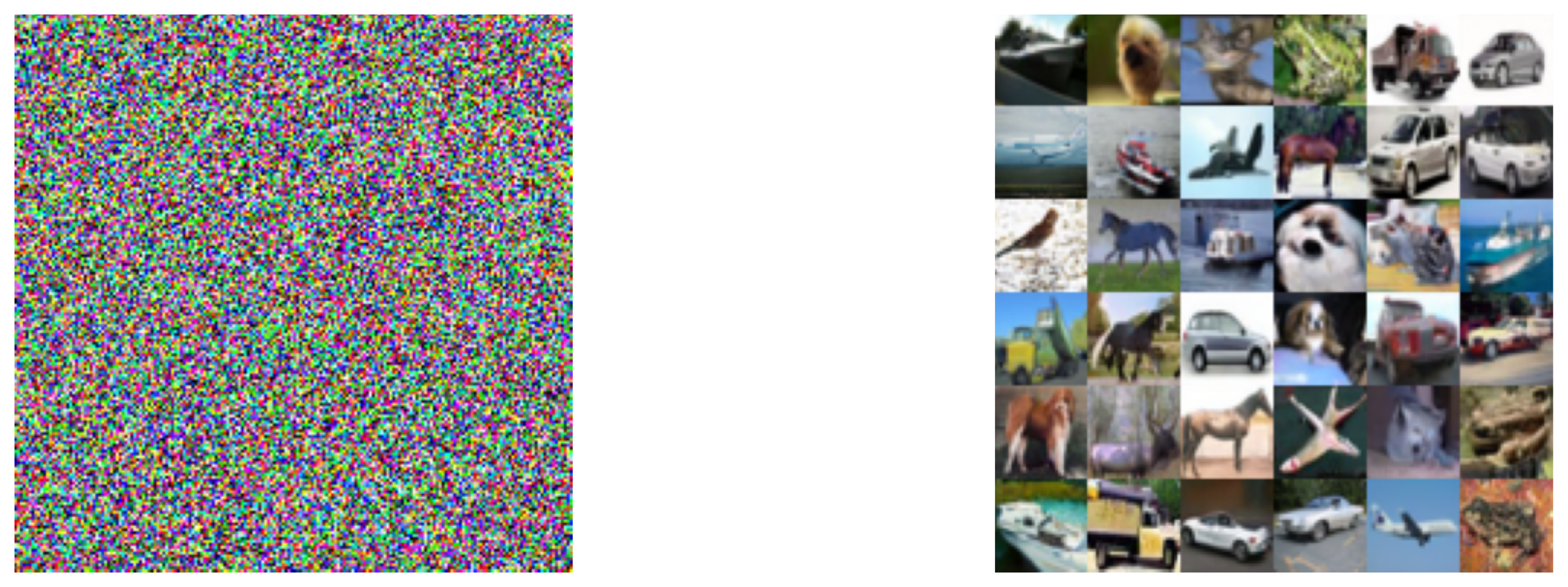}
    \caption{Images generated (right) by VP-SDE from the noise (left).}
    \label{fig:vpsde_generation_pictures}
\end{figure}
\begin{figure}[h]
    \centering
    \includegraphics[width=0.9\textwidth]{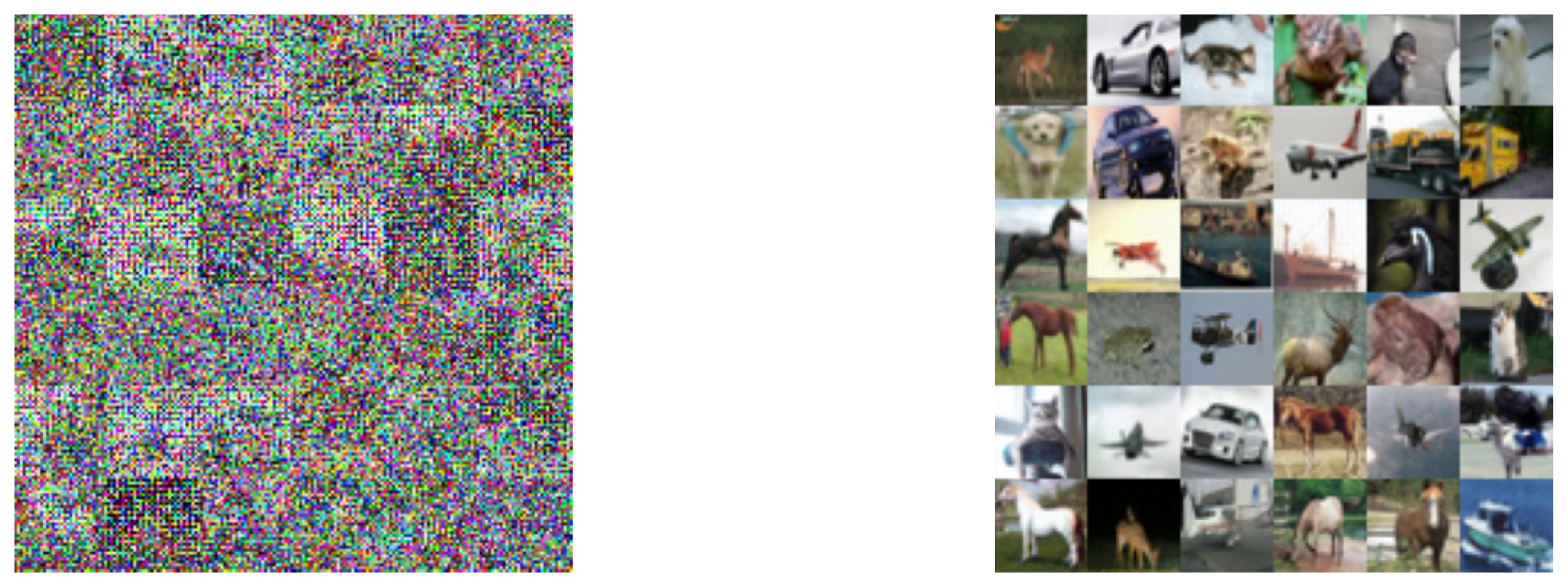}
    \caption{Images generated (right) by Action Matching from the lower resolution images (left).}
    \label{fig:superres_pictures}
\end{figure}
\begin{figure}[h]
    \centering
    \includegraphics[width=0.9\textwidth]{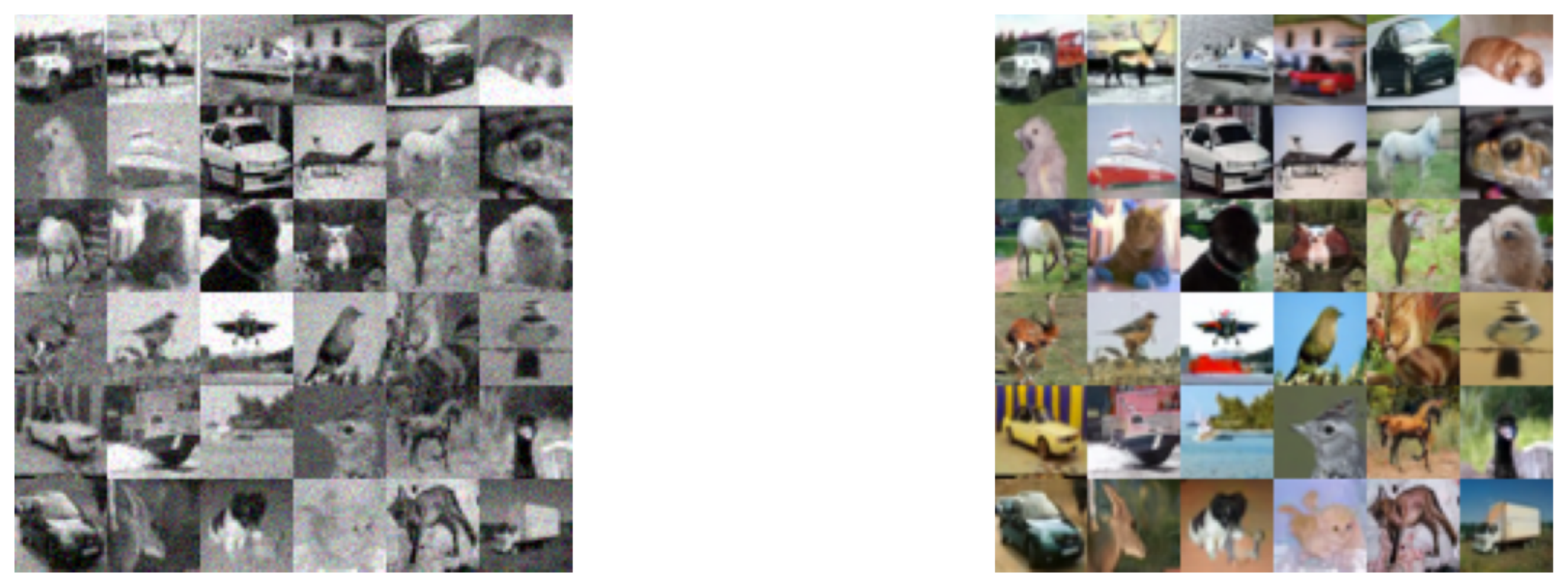}
    \caption{Images colored (right) by Action Matching from the grayscale images (left).}
    \label{fig:color_pictures}
\end{figure}

\end{document}